\documentclass[11pt]{article}
\usepackage[letterpaper, left=1in, right=1in, top=1in,bottom=1in]{geometry}

\usepackage{parskip}

\usepackage{booktabs} 

\usepackage[utf8]{inputenc}

\usepackage{geometry}

\usepackage{graphpap,amscd,mathrsfs,graphicx,lscape,enumitem,dsfont,bm,url,subfigure}
\usepackage{epsfig,amstext,xspace}
\usepackage{algorithm,comment}
\usepackage{algorithmic}
\usepackage{thmtools,thm-restate}
\usepackage{pifont}

\usepackage{amssymb}
\usepackage{amsmath}
\usepackage{amsthm}

\usepackage{algorithm,algorithmic}
\usepackage{auxiliary}
\usepackage[utf8]{inputenc} 
\usepackage[T1]{fontenc}    
\usepackage{hyperref}       
\usepackage{url}            
\usepackage{booktabs}       
\usepackage{amsfonts}       
\usepackage{nicefrac}       
\usepackage{microtype}      
\usepackage{framed}
\usepackage{paralist}
\usepackage{makecell}
\newcommand{\inner}[2]{\left\langle #1, #2 \right\rangle}

\usepackage[misc]{ifsym}

\usepackage[square]{natbib}

\newtheorem{discussion}{Discussion}

\usepackage{color}              
\usepackage[suppress]{color-edits}
\addauthor{tl}{cyan}
\addauthor{vm}{green}
\addauthor{rpl}{brown}

\usepackage{authblk}
\begin{document}
\title{Policy Optimization with Stochastic Mirror Descent \footnote{This submission has been accepted by AAAI 2022. L.Yang and Y.Zhang share an equal contribution for this work. \\L.Yang now is with Peking University, Email: \texttt{yanglong001@pku.edu.cn}. J. Wen is with Harvard Medical School, this work is done when he was at Zhejiang University, E-mail: \texttt{jungel2star@gmail.com}.}}

 \author[1]{Long Yang}
  \author[2]{Yu Zhang}
  \author[1]{ Gang Zheng }
\author[3]{Qian Zheng}
  \author[1]{Pengfei Li}
    \author[1]{Jianhang Huang}
     \author[1]{Jun Wen}
\author[1]{Gang Pan}
\affil[1]{College of Computer Science and Technology, Zhejiang University, China}
\affil[2]{Netease Games AI Lab, HangZhou, China}
\affil[3]{School of Electronics and Electrical Engineering, Nanyang Technological University, Singapore}
\affil[ ]{\textsuperscript{1}\texttt{\{yanglong,gpan,gang\_zheng,pfl,huangjianhang\}@zju.edu.cn}}
\affil[ ]{\textsuperscript{2}\texttt{zhangyu15@corp.netease.com}}
\affil[ ]{\textsuperscript{3}\texttt{zhengqian@ntu.edu.sg}}
\date{\today}
\maketitle
\begin{abstract}
Improving sample efficiency has been a longstanding goal in reinforcement learning.
This paper proposes $\mathtt{VRMPO}$ algorithm: a sample efficient policy gradient method with stochastic mirror descent.
In $\mathtt{VRMPO}$, a novel variance-reduced policy gradient estimator is presented to improve sample efficiency.
We prove that the proposed $\mathtt{VRMPO}$ needs only $\mathcal{O}(\epsilon^{-3})$ sample trajectories to achieve an $\epsilon$-approximate first-order stationary point, 
which matches the best sample complexity for policy optimization.
The extensive experimental results demonstrate that $\mathtt{VRMPO}$ outperforms the state-of-the-art policy gradient methods in various settings.
\end{abstract}

\section{Introduction}


Policy gradient \citep{williams1992simple,sutton2000policy} is widely used to search the optimal policy in reinforcement learning (RL), and it has achieved significant successes in challenging fields such as playing Go \citep{silver2016mastering,silver2017mastering} or robotics  \citep{duan2016benchmarking}.

However, policy gradient methods suffer from high sample complexity, since many existing popular methods require to collect a lot of samples for each step to update its parameters~\citep{mnih2016asynchronous,haarnoja2018soft,meng2019qualitative,xu2020sample,xinglearning2021},
which partially reduces the effectiveness of the samples.
Besides, it is still very challenging to provide a theoretical analysis of sample complexity for policy gradient methods instead of empirically improving sample efficiency.

To improve sample efficiency, this paper addresses how to design an efficient and convergent algorithm with stochastic mirror descent (SMD) \citep{nemirovsky1983problem}.
SMD keeps the advantage of low memory requirement and low computational complexity \citep{lei2018stochastic,yang2021sample}, which implies SMD needs less samples to learn a model.
However, the significant challenges of applying the existing SMD to RL are two-fold:
\textbf{1)} The objective of policy-based RL is a typical non-convex function,
\citep{ghadimi2016mini} show that it may cause instability and even divergence when updating the parameter of a non-convex objective by SMD via a single sample.
\textbf{2)} The large variance of policy gradient estimator is a critical bottleneck of improving sample efficiency for policy optimization with SMD.
The non-stationary sampling process with the environment will lead to a large variance on the policy gradient estimator \citep{Matteo2018Stochastic}, which requires more samples to get a robust policy gradient and results in poor sample efficiency \citep{liu2018action}.
To address the above challenges:

We provide a theory analysis of the dilemma of applying SMD to policy optimization.
Result (\ref{lower-bound-1}) shows that under the Assumption \ref{ass:On-policy-derivative}, deriving the algorithm directly via SMD can not
guarantee the convergence for policy optimization.
Furthermore, we propose a new algorithm $\mathtt{MPO}$ that keeps a provable convergence guarantee (see Theorem \ref{cr-theo1}).
Designing a new gradient estimator according to historical information of policy gradient is the key to $\mathtt{MPO}$.

We propose a variance-reduced mirror policy optimization algorithm ($\mathtt{VRMPO}$): an efficient sample method via constructing a variance reduced policy gradient estimator.
Concretely, we design an efficiently computable policy gradient estimator (see Eq.(\ref{Eq:G_t})) that utilizes fresh information and yields a more accurate estimation of the policy gradient,
which is the key to improve sample efficiency.
Theorem \ref{theo:CR-VRMP} illustrates that $\mathtt{VRMPO}$ needs $\mathcal{O}(\epsilon^{-3})$ sample trajectories to achieve an $\epsilon$-approximate first-order stationary point ($\epsilon$-FOSP). To our best knowledge, the proposed $\mathtt{VRMPO}$ matches the best sample complexity among the existing literature. 
Particularly, although $\mathtt{SRVR}$-$\mathtt{PG}$ \citep{xu2020sample,xu2021sample} achieves the same sample complexity as $\mathtt{VRMPO}$, 
our approach needs less assumptions than \citep{xu2020sample,xu2021sample}, and our $\mathtt{VRMPO}$ unifies $\mathtt{SRVR}$-$\mathtt{PG}$.
We have presented more comparisons and discussions in Remark \ref{remark:difference-vrmpo-hapg}. 
Besides, empirical result shows $\mathtt{VRMPO}$ converges faster than $\mathtt{SRVR}$-$\mathtt{PG}$.

\section{Background and Stochastic Mirror Descent}
\label{Background and Notations}

Reinforcement learning (RL) is often formulated as \emph{Markov decision processes} (MDP) ${M}=(\mathcal{S},\mathcal{A},{P},{R},\rho_0,\gamma)$, 
where $\mathcal{S}$ is state space, $\mathcal{A}$ is action space;
$P(s^{'}|s,a)$ is the probability of the state transition from $s$ to $s^{'}$ under playing $a$;
$R(\cdot,\cdot):\mathcal{S}\times\mathcal{A}\rightarrow [-R_{\max},R_{\max}]$ is the reward function, where $R_{\max}$ is a certain positive scalar.
$\rho_{0}(\cdot):\mathcal{S}\rightarrow[0,1]$ is the initial state distribution and $\gamma\in(0,1)$.

Policy $\pi_{\theta}(a|s)$ is a probability distribution on $\mathcal{S}\times\mathcal{A}$ with a parameter $\theta\in\mathbb{R}^p$.
Let
$\tau=\{s_{t}, a_{t}, r_{t+1}\}_{t=0}^{H_\tau} $ be a trajectory, 
where $s_{0}\sim\rho_{0}(s_0)$, $a_{t}\sim\pi_{\theta}(\cdot|s_{t}),r_{t+1}={R}(s_{t},a_{t})$, $s_{t+1}\sim P(\cdot|s_{t},a_{t})$, and $H_\tau$ is the finite horizon of $\tau$. 
The expected return function $J(\theta)$ is defined as follows,
\begin{flalign}
\label{Eq:J-theta}
    J(\theta)\overset{\text{def}}=\int_{\tau}P(\tau|\theta)R(\tau)\text{d}\tau=\mathbb{E}_{\tau\sim\pi_{\theta}}[R(\tau)],
\end{flalign}
where $P(\tau|\theta)=\rho_{0}(s_{0})\prod_{t=0}^{H_\tau}P({s_{t+1}|s_{t},a_{t})}\pi_{\theta}(a_{t}|s_{t})$ is the probability of generating $\tau$, $R(\tau)=\sum_{t=0}^{H_\tau}\gamma^{t}r_{t+1}$ is the accumulated discounted return. 
Let $\mathcal{J}(\theta)=:-J(\theta),$
the central problem of policy-based RL is to solve the problem:
\begin{flalign}
\label{Eq:thata-optimal}
    \theta^{\star}=\arg\max_{\theta}J(\theta)\Longleftrightarrow\theta^{\star}=\arg\min_{\theta}\mathcal{J}(\theta).
\end{flalign}
Computing $\nabla J(\theta)$ analytically, we have
 \begin{flalign}
    \label{Eq:PG-theorem-2}
     \nabla J(\theta)=\mathbb{E}_{\tau\sim\pi_{\theta}}\left[\sum_{t\ge0}\nabla_{\theta} \log \pi_{\theta}(a_{t}|s_{t})R(\tau)\right].
 \end{flalign}
Let 
$g(\tau|\theta)=\sum_{t=0}^{H_\tau}
     \nabla_{\theta}\log\pi_{\theta}(a_{t}|s_{t})R(\tau)$, which is an unbiased estimator of $\nabla J(\theta)$.
 {Vanilla policy gradient} ($\mathtt{VPG}$) is a straightforward way to solve problem (\ref{Eq:thata-optimal}) as follows,
\[
\theta\leftarrow\theta+\alpha g(\tau|\theta),
\]
where $\alpha$ is step size.
\begin{assumption}\emph{\citep{Matteo2018Stochastic}}
    \label{ass:On-policy-derivative}
    For each pair $(s, a)\in\mathcal{S}\times\mathcal{A}$, $\theta\in\mathbb{R}^p$, and all components $i$, $j$, there exists positive constants $G$, $F$ such that:
    \begin{flalign}
    \label{def:F-G}
    |\nabla_{\theta_i}\log\pi_{\theta}(a|s)|\leq G, ~~~~\left|\dfrac{\partial^{2}}{\partial\theta_i\partial\theta_j}\log\pi_{\theta}(a|s)\right|\leq F.
    \end{flalign}
\end{assumption}
Assumption \ref{ass:On-policy-derivative} implies $\nabla J(\theta)$ is $L$-Lipschiz \cite[Lemma B.2]{Matteo2018Stochastic}, i.e.,
\begin{flalign}
\label{def:L}
\|\nabla J(\theta_{1})-\nabla J(\theta_{2})\|\leq L\|\theta_{1}-\theta_{2}\|,
\end{flalign}
where
$
 L=R_{\max}H_{\tau}(H_{\tau}G^{2}+F)/(1-\gamma),
$
Besides, under Assumption \ref{ass:On-policy-derivative}, \cite{shen2019hessian} have shown the property:
\begin{flalign}
\label{def:sigma}
\|g(\tau|\theta)-\nabla J(\theta)\|_{2}^{2}\leq\dfrac{G^{2}R_{\max}^{2}}{(1-\gamma)^4}=:\sigma^2.
\end{flalign}

\subsection{SMD and Bregman Gradient}
Now, we review some basic concepts of stochastic mirror descent(SMD) and Bregman gradient.

Let's consider the stochastic optimization problem,
\begin{flalign}
\label{min-objective}
    \min_{\theta\in D_\theta}\{f(\theta)=\mathbb{E}[F(\theta;\xi)]\},
\end{flalign}
where $D_\theta\in\mathbb{R}^{n}$  is a nonempty convex compact set,
$\xi$ is a random vector whose probability distribution $\mu$ is supported on $\Xi\in\mathbb{R}^{d}$ and $F:D_\theta \times \Xi \rightarrow \mathbb{R}$. We assume that the expectation
$\mathbb{E}[F(\theta;\xi)]=\int_{\Xi}F(\theta;\xi)\text{d}\mu(\xi)$ is well defined and finite-valued for every $\theta \in D_{\theta}$.
\begin{definition}[Proximal Operator]
\label{def1} 
Let $T$ be defined on a closed convex $\mathcal{X}$, and $\alpha>0$. The proximal operator of $T$ is
\begin{flalign}
\label{def:Moreau-envelope}
\mathcal{M}^{\psi}_{\alpha,T}(z)=\underset{x\in{\mathcal{X}}}{\arg\min}\left\{T(x)+\dfrac{1}{\alpha}D_{\psi}(x,z)\right\},
\end{flalign}
where $\psi(\cdot)$ is a continuously-differentiable, $\zeta$-strictly convex function satisfies 
$\langle x-y,\nabla \psi(x)-\nabla \psi(y)\rangle\ge\zeta\|x-y\|^{2}, \zeta>0,$
 $D_{\psi}(\cdot,\cdot)$ is 
 Bregman distance: $\forall~x,y\in\mathcal{X}$,
\[D_{\psi}(x,y)=\psi(x)-\psi(y)-\langle \nabla \psi(y),x-y \rangle.\]
\end{definition}

\textbf{Stochastic Mirror Descent (SMD)}. The SMD solves (\ref{min-objective}) by
generating an iterative solution as follows,
\begin{flalign}
        \label{Eq:SMD}  
       \theta_{t+1}
        =\mathcal{M}^{\psi}_{\alpha_{t},\ell(\theta)}(\theta_{t})
        =\arg\min_{\theta\in D_\theta}\left\{\langle g_t,\theta\rangle+\dfrac{1}{\alpha_{t}}D_{\psi}(\theta,\theta_{t})\right\},
\end{flalign}
where $\alpha_{t}>0$ is step-size, $\ell(\theta)=\langle g_t,\theta\rangle$ is the first-order approximation of $f(\theta)$ at $\theta_{t}$,
 $g_{t}=g(\theta_{t},\xi_{t})$ is stochastic subgradient such that $g(\theta_{t})=\mathbb{E}[g(\theta_{t},\xi_{t})]\in\partial f(\theta)|_{\theta=\theta_{t}}$,
$\{\xi_{t}\}_{t\ge0}$ represents a draw form distribution $\mu$, 
and
$\partial f(\theta)=\{g| f(\theta)-f(\omega)\leq g^{\top}(\theta-\omega) ,\forall \omega \in \textbf{dom}(f)\}$.
If we choose $\psi(x)=\dfrac{1}{2}\|x\|^{2}_{2}$, then $D_\psi(x,y)=\dfrac{1}{2}\|x-y\|_{2}^{2}$, since then iteration
(\ref{Eq:SMD}) is reduced to stochastic gradient decent (SGD).

\textbf{Convergence Criteria: Bregman Gradient}.
Recall $\mathcal{X}$ is a closed convex set on $\mathbb{R}^{n}$,
$\alpha>0$, $T(x)$ is defined on $\mathcal{X}$. 
The Bregman gradient of $T$ at $x\in\mathcal{X}$ is defined as:
    \begin{flalign}
        \label{bregman-grdient-mapping}
        \mathcal{G}_{\alpha,T}^{\psi}(x)=\alpha^{-1}(x-\mathcal{M}^{\psi}_{\alpha, T}(x)),
    \end{flalign}
where $\mathcal{M}^{\psi}_{\alpha, T}(\cdot)$ is defined in Eq.(\ref{def:Moreau-envelope}).
If $\psi(x)=\dfrac{1}{2}\|x\|^{2}_{2}$, according to \citep[Theorem 27.1]{bauschke2011convex},
 then $x^{\star}$ is a critical point of $T$ if and only if $\mathcal{G}_{\alpha,T}^{\psi}(x^{\star})=\nabla T(x^{\star})=0$.
Thus, Bregman gradient (\ref{bregman-grdient-mapping}) is a generalization of standard gradient.
Remark \ref{remark-1-discuss} provides us some insights to understand Bregman gradient as a convergence criterion.
\begin{remark}
\label{remark-1-discuss}
Let $T(\cdot)$ be a convex function, according to \citep[Proposition 5.4.7]{bertsekas2009convex}: $x^{\star}$ is a stationarity point of $T(\cdot)$ if and only if 
\begin{flalign}
\label{criteria}
0\in\partial (T+\delta_{\mathcal{X}})(x^{\star}),
\end{flalign} 
where $\delta_{\mathcal{X}}(\cdot)$ is the indicator function on $\mathcal{X}$. 
Furthermore, if $\psi(x)$ is twice continuously differentiable,
let $\tilde{x}=\mathcal{M}^{\psi}_{\alpha, T}(x)$, by the definition of $\mathcal{M}^{\psi}_{\alpha, T}(\cdot)$ (\ref{def:Moreau-envelope}), we have
\begin{flalign}
 \label{Eq:optimal-comdition-approximal}
    0\in\partial (T+\delta_{\mathcal{X}})(\tilde{x})+\big(\nabla\psi(\tilde{x})-\nabla\psi(x)\big)
    \overset{(*)}\approx \partial (T+\delta_{\mathcal{X}})(\tilde{x}) +\alpha\mathcal{G}_{\alpha,T}^{\psi}(x)\nabla^{2}\psi(x),
\end{flalign}
Eq.($*$) holds due to Taylor expansion of $\nabla\psi(x)$ on first order.
If $\mathcal{G}_{\alpha,T}^{\psi}(x)\approx0$, 
Eq.(\ref{Eq:optimal-comdition-approximal}) implies the origin point $0$ is near the set $\partial (T+\delta_{\mathcal{X}})(\tilde{x})$, i.e., according to the criteria (\ref{criteria}), $\tilde{x}$ is close to a stationary point.
For the iteration (\ref{Eq:SMD}), we focus on the time when it makes the $\mathcal{G}_{\alpha,T}^{\psi}(\theta_{t})$ near origin point $\bm{0}$.
Formally, we are satisfied with finding an $\epsilon$-approximate first-order stationary point ($\epsilon$-FOSP) $\theta_{\epsilon}$ such that
 \begin{flalign}
        \label{epsilon-FOSP}
        \|\mathcal{G}_{\alpha,T(\theta_{\epsilon})}^{\psi}(\theta_{\epsilon})\|_2\leq\epsilon.
    \end{flalign}
Particularly, for policy optimization (\ref{Eq:thata-optimal}),
we would choose $T(\theta)=\inner{-\nabla J(\theta_t)}{\theta}$.
\end{remark}

\section{Stochastic Mirror Policy Optimization}
\label{MPO-sec}

In this section, we solve the problem (\ref{Eq:thata-optimal}) via SMD. 
Firstly, we analyze the theoretical dilemma of applying SMD directly to policy optimization, 
and result shows that under the common Assumption \ref{ass:On-policy-derivative}, there still lacks a provable guarantee of solving (\ref{Eq:thata-optimal}) via SMD directly.
Then, we propose a convergent mirror policy optimization algorithm ($\mathtt{MPO}$).

\subsection{Theoretical Dilemma}

For each $k \in [1,N-1]$, $\tau_{k}=\{s_{t}, a_{t}, r_{t+1}\}_{t=0}^{H_{\tau_{k}}}\sim\pi_{\theta_k}$, and we receive the gradient information as follows, 
\begin{flalign}
\label{def:g_k}
-g(\tau_{k}|\theta_{k})=-\sum_{t\ge0}
\nabla_{\theta}\log\pi_{\theta}(a_{t}|s_{t})R(\tau_k)|_{\theta=\theta_k}.
\end{flalign}
According to (\ref{Eq:SMD}), we define the update rule as follows,
 \begin{flalign}
\label{Eq:PP-MD}
\theta_{k+1}=\mathcal{M}^{\psi}_{\alpha_{k},\langle -g(\tau_{k}|\theta_{k}),\theta\rangle}(\theta_{k})
=\arg\min_{\theta}\left\{\langle -g(\tau_{k}|\theta_{k}),\theta\rangle+\dfrac{1}{\alpha_{k}}D_{\psi}(\theta,\theta_{k})\right\},
\end{flalign}
where $\alpha_{k}$ is step-size.
After ($N-1$) episodes, we receive a collection $\{\theta_{k}\}_{k=1}^{N}$.
Since $-J(\theta)$ is non-convex,
according to Ghadimi, et al \citep{ghadimi2016mini}, a standard strategy for analyzing non-convex optimization is to pick up the output $\tilde{\theta}_{N}$ from the following distribution (\ref{mass-distriution}) over $\{1,2,\cdots,N\}$:
\begin{flalign}
\label{mass-distriution}
\mathbb{P}(\tilde{\theta}_{N}=\theta_{k})=\dfrac{\zeta\alpha_{k}-L\alpha^{2}_{k}}{\sum_{i=1}^{N}(\zeta\alpha_{i}-L\alpha^{2}_{i})}, k\in[1,N],
\end{flalign}
where step-size $\alpha_{k}\in(0,{\zeta}/{L})$.
\begin{theorem}
\label{theorem-mpo-1}
\emph{\citep{ghadimi2016mini}}
Under Assumption \ref{ass:On-policy-derivative}, consider the sequence $\{\theta_{k}\}_{k=1}^{N}$ generated by (\ref{Eq:PP-MD}), the output $\tilde{\theta}_{N}=\theta_{k}$ follows the distribution (\ref{mass-distriution}). Let $\ell(g,u)=\langle g,u\rangle$, $g_k=(\tau_{k}|\theta_{k})$, Let $\Delta=J(\theta^{\star})-J(\theta_1)$. Then, 
    \begin{flalign}
 \label{CR-mpo-1}
    \mathbb{E}\left[\|\mathcal{G}_{\alpha_k,\ell(-{g}_k,{\theta}_{k})}^{\psi}(\tilde{\theta}_{N})\|_2^{2}\right]    
        \leq\dfrac{\Delta+{\sigma^{2}}/{\zeta}\sum_{i=1}^{N}{\alpha_i}}{{\sum_{i=1}^{N}(\zeta\alpha_{i}-L\alpha^{2}_{i})}}.
    \end{flalign}   
\end{theorem}
Unfortunately, the lower bound of (\ref{CR-mpo-1}) reaches
\begin{flalign}
\label{lower-bound-1}
\dfrac{J(\theta^{\star})-J(\theta_1){+{\sigma^{2}}/{\zeta}\sum_{i=1}^{N}{\alpha_i}}}{{\sum_{i=1}^{N}(\zeta\alpha_{i}-L\alpha^{2}_{i})}}\ge\dfrac{\sigma^{2}}{\zeta^{2}},
\end{flalign}
which can not guarantee the convergence of (\ref{Eq:PP-MD}), no matter how the step-size $\alpha_{k}$ is specified.
Thus, under Assumption \ref{ass:On-policy-derivative}, updating parameters according to (\ref{Eq:PP-MD}) and the output following (\ref{mass-distriution}) lacks a provable convergence guarantee.

\begin{discussion}[Open Problems]

Eq.(\ref{Eq:PP-MD}) is a general rule that unifies many existing algorithms.
If $\psi(\theta)=\dfrac{1}{2}\|\theta\|^{2}_{2}$, then (\ref{Eq:PP-MD}) is $\mathtt{VPG}$~\cite{williams1992simple}.
The update (\ref{Eq:PP-MD}) is natural policy gradient~\cite{kakade2002natural} if we choose $\psi(\theta)=\dfrac{1}{2}\theta^{\top}F(\theta)\theta$, where $F(\theta)=\mathbb{E}_{\tau\sim\pi_{\theta}}[\nabla_{\theta}\log \pi_{\theta}(s,a)\nabla_{\theta}\log \pi_{\theta}(s,a)^{\top}]$ is Fisher information matrix.
If $\psi$ is Boltzmann-Shannon entropy,
then $D_\psi$ is KL divergence and update (\ref{Eq:PP-MD}) is reduced to relative entropy policy search~\citep{peters2010relative,tomar2020mirror}.
Despite extensive works around above methods, existing works are scattered and fragmented in both theoretical and empirical aspects \citep{agarwal2020optimality}.
Thus, it is of great significance to establish the fundamental theoretical convergence properties of iteration (\ref{Eq:PP-MD}):
\begin{center}
\emph{\textbf{What conditions guarantee the convergence of (\ref{Eq:PP-MD})?}}
\end{center}
This is an open problem. From the previous discussion, intuitively, the iteration (\ref{Eq:PP-MD}) is a convergent scheme since
particular mirror maps $\psi$ can lead (\ref{Eq:PP-MD}) to some popular empirically effective policy-based RL algorithms, but there still lacks a complete theoretical convergence analysis of (\ref{Eq:PP-MD}).
\end{discussion}

\subsection{MPO: A Convergent  Implementation}

In this section, we propose a convergent mirror policy optimization ($\mathtt{MPO}$) as follows, for each step $k$:
\begin{flalign}
\label{Eq:PP-MD-averge}
\theta_{k+1}&=\mathcal{M}^{\psi}_{\alpha_{k},\langle -\hat{g}_{k},\theta\rangle}(\theta_{k})=\arg\min_{\theta\in\Theta}\left\{\langle -\hat{g}_{k},\theta\rangle+\dfrac{1}{\alpha_{k}}D_{\psi}(\theta,\theta_{k})\right\},
\end{flalign}
where $\hat{g}_{k}$ is an arithmetic mean of previous episodes' gradient estimate $\{g(\tau_{i}|\theta_{i})\}_{i=1}^{k}$:
\begin{flalign}
\label{hat_g_k}
\hat{g}_{k} = \dfrac{1}{k}\sum^{k}_{i=1}g(\tau_{i}|\theta_{i}).
\end{flalign}
We present the details of an implementation of $\mathtt{MPO}$ in Algorithm \ref{alg:Mirror-Policy-Algorithm}.
Eq.(\ref{Eq:algo1-2}) is an incremental implementation of the average (\ref{hat_g_k}), thus, (\ref{Eq:algo1-2}) enjoys a lower storage cost than (\ref{hat_g_k}).

For a given episode, the gradient flow (\ref{hat_g_k})/(\ref{Eq:algo1-2}) of $\mathtt{MPO}$ is slightly different from the traditional $\mathtt{VPG}$, $\mathtt{REINFORCE}$ \citep{williams1992simple}, or $\mathtt{DPG}$ \citep{silver2014deterministic} whose gradient estimator (\ref{def:g_k}) follows the current episode, while our $\mathtt{MPO}$ uses an arithmetic mean of all the previous policy gradients.
 The gradient estimator (\ref{def:g_k}) is a natural way to estimate the term \[-\nabla J(\theta_t)=-\mathbb{E}\left[\sum_{k=0}^{H_{\tau_{t}}}
     \nabla_{\theta}\log\pi_{\theta}(a_{k}|s_{k})R(\tau_t)\right],\] i.e., using the current trajectory to estimate policy gradient.
     Unfortunately, under Assumption \ref{ass:On-policy-derivative},  the result of (\ref{lower-bound-1}) shows using (\ref{def:g_k}) with SMD lacks a guarantee of convergence.
This is exactly the reason why we abandon the way (\ref{def:g_k}) and turn to propose (\ref{hat_g_k})/(\ref{Eq:algo1-2}) to estimate policy gradient. 

\begin{algorithm}[t]
    \caption{MPO}
    \label{alg:Mirror-Policy-Algorithm}
    \begin{algorithmic}[1]
        \STATE {\bfseries Initialize:} parameter $\theta_{1}$, step-size$\alpha_{k}> 0$, $g_{0}=0$, parametric policy $\pi_{\theta}(a|s)$, and map $\psi$.
        \FOR{$k=1$ {\bfseries to} $N$}
        \STATE Generate a trajectory $\tau_{k}=\{s_{t}, a_{t}, r_{t+1}\}_{t=0}^{H_{\tau_k}} $
        $\sim$ $\pi_{\theta_{k}}$, temporary variable $g_0=0$.
        \vspace{-5pt}
        \begin{flalign}
        \label{Eq:algo1-1}
        g_{k} &\leftarrow \sum_{t=0}^{H_{\tau_{k}}}
  \nabla_{\theta}\log\pi_{\theta}(a_{t}|s_{t})R(\tau_{k})|_{\theta=\theta_k}\\
        \label{Eq:algo1-2}
        \hat{g}_{k}&\leftarrow \dfrac{1}{k}g_{k}+(1-\dfrac{1}{k}) \hat{g}_{k-1}\\
        \label{Eq:algo1-3}
        \theta_{k+1}&\leftarrow\arg\min_{\omega}\left\{\langle -\hat{g}_k,\omega\rangle+\dfrac{1}{\alpha_k}D_{\psi}(\omega,\theta_k)\right\}
        \end{flalign}
        \vspace{-16pt}    
        \ENDFOR
        \STATE {\bfseries Output} $\tilde{\theta}_{N}$ according to (\ref{mass-distriution}).
    \end{algorithmic}
\end{algorithm}

\begin{theorem}[Convergence of Algorithm \ref{alg:Mirror-Policy-Algorithm}]
    \label{cr-theo1}
    Under Assumption \ref{ass:On-policy-derivative}, and the total trajectories are $\{\tau_{k}\}_{k=1}^{N}$.
    Consider the sequence $\{\theta_{k}\}_{k=1}^{N}$ generated by Algorithm \ref{alg:Mirror-Policy-Algorithm}, and the output $\tilde{\theta}_{N}=\theta_{n}$ follows the distribution of (\ref{mass-distriution}).
 Let $0<\alpha_{k}<\dfrac{\zeta}{L}$,
$\ell(g,u)=\langle g,u\rangle$, $\hat{g}_{k}=\dfrac{1}{k}\sum_{i=1}^{k}g_i$, and $\Delta=J(\theta^{\star})-J(\theta_1)$, where 
  ${g}_{i}=\sum_{t=0}^{H_{\tau_{i}}}
  \nabla_{\theta}\log\pi_{\theta}(a_{t}|s_{t})R(\tau_{i})|_{\theta=\theta_i}$.
   Then the output $\tilde{\theta}_{N}=\theta_{n}$ satisfies
    \begin{flalign}
\label{CR-1}
        \mathbb{E}\left[\|\mathcal{G}_{\alpha_n,\ell(-{g}_n,{\theta}_{n})}^{\psi}({\theta}_{n})\|_2^{2}\right]
         \leq\dfrac{\Delta{+{\sigma^{2}}/{\zeta}\sum_{k=1}^{N}{\dfrac{\alpha_k}{k}}}}{{\sum_{k=1}^{N}(\zeta\alpha_{k}-L\alpha^{2}_{k})}}.
    \end{flalign}
\end{theorem}
For the proof, see Appendix \ref{sec: app-A}.

Let $\alpha_{k}=\dfrac{\zeta}{2L}$,  then
\[
\mathbb{E}\left[\|\mathcal{G}_{\alpha_n,\ell(-\hat{g}_n,\theta_n)}^{\psi}(\theta_n)\|^{2}\right]\leq
\dfrac{4L\Delta{+2{\sigma^{2}}\sum_{k=1}^{N}{\dfrac{1}{k}}}}{N\zeta^2}=\mathcal{O}\left(\dfrac{\ln N}{N}\right).
\]
Our scheme of $\mathtt{MPO}$ partially answers the previous open problem through conducting a new policy gradient.

\section{VRMPO: Variance Reduction Mirror Policy Optimization}
\label{VR}
In this section, we propose a variance reduction version of $\mathtt{MPO}$: $\mathtt{VRMPO}$.
Inspired by the above work of \cite{nguyen2017sarah}, we provide an efficiently computable policy gradient estimator; then, we prove that the $\mathtt{VRMPO}$ needs $\mathcal{O}(\epsilon^{-3})$ sample trajectories to achieve an $\epsilon$-FOSP that matches the best sample complexity.

\subsection{Methodology}
For any initial $\theta_{0}$, let $\{\tau^{0}_{j}\}_{j=1}^{N}\sim\pi_{\theta_0}$, we estimate the initial policy gradient as follows,
\begin{flalign}
\label{def:nabla-N}
G_{0}=-\hat{\nabla}_{N}J(\theta_{0})\overset{\text{def}}=-\dfrac{1}{N}\sum_{j=1}^{N}g(\tau^{0}_{j}|\theta_{0}).
\end{flalign}
Let $\theta_1=\theta_0-\alpha G_0$,
for each step $k\in\mathbb{N}^{+}$, let $\{\tau^{k}_{j}\}_{j=1}^{N}$ be the trajectories generated by $\pi_{\theta_k}$, we define the policy gradient estimator $G_{k}$ and update rule as follows,
\begin{flalign}
    \label{Eq:G_t}
    G_k=G_{k-1}+\dfrac{1}{N}\sum_{j=1}^{N}\left(-g(\tau^{k}_j|\theta_{k})+g(\tau^{k}_j|\theta_{k-1})\right),\\
    \label{theta-mirror}
   \theta_{k+1}=\arg\min_{\theta}\left\{\langle G_{k},\theta\rangle+\dfrac{1}{\alpha}D_{\psi}(\theta,\theta_{k})\right\}.
    \end{flalign}
In (\ref{Eq:G_t})$, -g(\tau^{k}_j|\theta_{k})$ and $g(\tau^{k}_j|\theta_{k-1})$ share the same trajectory $\{\tau^{k}_j\}_{j=1}^{N}$, which plays a critical role in reducing the variance of gradient estimator~\citep{shen2019hessian}.
Besides, it is different from (\ref{hat_g_k}), we admit a simple recursive formulation to conduct the gradient estimator, see (\ref{Eq:G_t}),
which captures the technique from $\mathtt{SARAH}$ \citep{nguyen2017sarah}.
For each step $k$,
the term \[\dfrac{1}{N}\sum_{j=1}^{N}\left(-g(\tau^{k}_j|\theta_{k})+g(\tau^{k}_j|\theta_{k-1})\right)\] can be seen as an additional ``noise" for the policy gradient estimate. 
A lot of practices show that conducting a gradient estimator with such additional ``noise'' enjoys a lower variance and speeding up the convergence~\citep{reddi2016stochastic}.
More details are shown in Algorithm \ref{alg:SVR-MPD}.

\begin{algorithm}[tb]
    \caption{VRMPO.}
    \label{alg:SVR-MPD}
    \begin{algorithmic}[1]
        \STATE {\bfseries  Initialize:} Policy $\pi_{\theta}(a|s)$ with parameter $\tilde{\theta}_{0}$, mirror map $\psi$, step-size $\alpha> 0$, epoch size $K$,$m$.
        \FOR{$k=1$ {\bfseries to} $K$}
        \vspace{0.1pt}
        \STATE $\theta_{k,0}=\tilde{\theta}_{k-1}$, generate $\mathcal{T}_k=\{\tau_{i}\}_{i=1}^{N_1}$ $\sim$ $\pi_{{\theta}_{k,0}}$
        \vspace{0.1pt}    
        \STATE $\theta_{k,1}=\theta_{k,0}-\alpha G_{k,0}$,~~~where $G_{k,0}=-\hat{\nabla}_{N_1}J(\theta_{k,0})=-\dfrac{1}{N_1}\sum_{i=1}^{N_1}g(\tau_{i}|\theta_{k,0})$.
        \vspace{0.1pt}
        \FOR{$t=1$ {\bfseries to} $m-1$}
        \vspace{0.1pt}
        \STATE Generate $\{\tau_{j}\}_{j=1}^{N_2}\sim\pi_{{\theta}_{k,t}}$
        \vspace{0.5pt}
        \begin{flalign}
        \vspace{3pt}
        \label{algo:VRMPG-1}
        G_{k,t}&=G_{k,t-1}+\dfrac{1}{N_2}\sum_{j=1}^{N_2}(-g(\tau_j|\theta_{k,t})+g(\tau_j|\theta_{k,t-1})),\\
         \label{algo:VRMPG-2}
        \vspace{0.5pt}
        \theta_{k,t+1}&=\arg\min_{\omega}\left\{\langle G_{k,t},\omega\rangle+\dfrac{1}{\alpha}D_{\psi}(\omega,\theta_{k,t})\right\}.
        \vspace{0.3pt}
        \end{flalign}
        \ENDFOR
        \vspace{0.5pt}
        \STATE $\tilde{\theta}_{k}=\theta_{k,t}$ with $t$ chosen uniformly randomly from $\{0, 1, . . . , m\}$.
        \vspace{0.1pt}
        \ENDFOR
        \vspace{0.1pt}
        \STATE {\bfseries Output:} $\tilde{\theta}_{K}$.
    \end{algorithmic}
\end{algorithm}

\begin{theorem}[Convergence Analysis]
    \label{theo:CR-VRMP}
Consider $\{\tilde{\theta}_{k}\}_{k=1}^{K}$ generated by Algorithm \ref{alg:SVR-MPD}.
Under Assumption \ref{ass:On-policy-derivative}, and let $\zeta>\dfrac{5}{32}$. For any positive scalar $\epsilon$, let batch size of the trajectories of the outer loop 
\begin{flalign}
\nonumber
N_{1}&={\left(\dfrac{1}{8L\zeta^2} +\dfrac{1}{2(\zeta-\dfrac{5}{32})}\left(1+\dfrac{1}{32\zeta^2}\right)\right)}{\dfrac{\sigma^2}{\epsilon^2}},\\
\nonumber
m-1&=N_{2}={\sqrt{\left(\dfrac{1}{8L\zeta^2} +\dfrac{1}{2(\zeta-\dfrac{5}{32})}\left(1+\dfrac{1}{32\zeta^2}\right)\right)}}{\dfrac{\sigma}{\epsilon}},
\end{flalign}
the outer loop times \[K=\dfrac{8L(\mathbb{E} [\mathcal{J}(\tilde{\theta}_{0})] - \mathcal{J}(\theta^{\star}))(1+\dfrac{1}{16\zeta^2})}{{\sqrt{\big(\dfrac{1}{8L\zeta^2} +\dfrac{1}{2(\zeta-\dfrac{5}{32})}\left(1+\dfrac{1}{32\zeta^2}\right)}}\left(\zeta-\dfrac{5}{32}\right)}\dfrac{\sigma}{\epsilon},\] 
and step size $\alpha=\dfrac{1}{4L}$.
Then, Algorithm \ref{alg:SVR-MPD} outputs $\tilde{\theta}_{K}$ satisties
\begin{flalign}
\label{cr-vrmpo-bound}
\mathbb{E}\left[\|\mathcal{G}^{\psi}_{\alpha,\inner{-\nabla J(\tilde{\theta}_{K})}{\theta}}(\tilde{\theta}_{K})\|\right]\leq\epsilon.
\end{flalign}
\end{theorem}
For its proof, see Appendix \ref{app: proof-of-vrmpo}.

Theorem \ref{theo:CR-VRMP} illustrates that $\mathtt{VRMPO}$ needs 
\begin{flalign}
\nonumber
&K(N_1+(m-1)N_2)\\
\nonumber
=&\dfrac{8L(\mathbb{E} [\mathcal{J}(\tilde{\theta}_{0})] - \mathcal{J}(\theta^{*}))}{(\zeta-\dfrac{5}{32})}\left(1+\dfrac{1}{16\zeta^2}\right)\left(1+{\sqrt{\left(\dfrac{1}{8L\zeta^2} +\dfrac{1}{2(\zeta-\dfrac{5}{32})}\left(1+\dfrac{1}{32\zeta^2}\right)\right)}}{\dfrac{\sigma}{\epsilon}}\right)\dfrac{1}{\epsilon^2}
=\mathcal{O}\left(\dfrac{1}{\epsilon^3}\right)
\end{flalign}
random trajectories to achieve the $\epsilon$-FOSP. 
As far as we know, our $\mathtt{VRMPO}$ matches the best sample complexity as $\mathtt{HAPG}$~\citep{shen2019hessian} and $\mathtt{SRVR}$-$\mathtt{PG}$ \citep{xu2020sample,xu2021sample}.
In fact, according to \citet{shen2019hessian}, $\mathtt{REINFORCE}$ needs $\mathcal{O}(\epsilon^{-4})$ random trajectories to achieve the $\epsilon$-FOSP, and no provable improvement on its complexity has been made so far.
The same order of sample complexity of $\mathtt{REINFORCE}$ is shown by \cite{xu2019improved}.
With the additional assumptions
\[\mathbb{V}{\text{ar}}\left[\prod_{h=0}^{H}\dfrac{\pi_{\theta_0}(a_h|s_h)}{\pi_{\theta_t}(a_h|s_h)}\right],~~~\mathbb{V}{\text{ar}}[g(\tau|\theta)]<+\infty,\] 
\cite{Matteo2018Stochastic} show that the $\mathtt{SVRPG}$ achieves the
sample complexity of $\mathcal{O}(\epsilon^{-4})$.
Later, under the same assumption as Papini et al. \cite{Matteo2018Stochastic},  Xu et al. \cite{xu2019improved} reduce the sample complexity of $\mathtt{SVRPG}$ to $\mathcal{O}(\epsilon^{-\dfrac{10}{3}})$.
We summarize it in Table \ref{comparing-complex}.
\begin{remark}
\label{remark:difference-vrmpo-hapg}
It's remarkable that although our $\mathtt{VRMPO}$ shares sample complexity with $\mathtt{HAPG}$, $\mathtt{SRVR}$-$\mathtt{PG}$, and $\mathtt{VR}$-$\mathtt{BGPO}$\citep{huang2021bregman},
the difference between our $\mathtt{VRMPO}$ and theirs are at least three aspects: 
\begin{compactenum}[\ding{182}]
\item  \cite{shen2019hessian} derive their $\mathtt{HAPG}$ from the information of Hessian policy, our $\mathtt{VRMPO}$ provides a simple recursive formulation
to conduct the gradient estimator.
\end{compactenum}
\begin{compactenum}[\ding{183}]
\item If the mirror map $\psi$ is reduced to the $\ell_2$-norm, then $\mathtt{VRMPO}$ is $\mathtt{SRVR}$-$\mathtt{PG}$ exactly, i.e., $\mathtt{VRMPO}$ unifies $\mathtt{SRVR}$-$\mathtt{PG}$.
From Table \ref{comparing-complex}, we see $\mathtt{VRMPO}$ needs less conditions than \cite{xu2020sample} to achieve the same sample complexity.
\end{compactenum}
\begin{compactenum}[\ding{184}]
\item \citet{shen2019hessian}, \citet{xu2020sample} and \citet{huang2021bregman} only provide an off-line (i.e., Monte Carlo) policy gradient estimator, which is limited in complex domains. 
We have provided an on-line version of $\mathtt{VRMPO}$, and discuss some insights of practical tracks to the application to the complex domains, please see the section of experiment on MuJoCo task, Appendix E.1 and Algorithm \ref{on-line VRMPO}.
\end{compactenum}
\end{remark}

 \begin{table}[t]
\centering
 \label{table-1}  
 \footnotesize{
    \begin{tabular}{|c|c|c|}
         \hline
        Algorithm  &Conditions& Complexity \\
         \hline
        {$\mathtt{VPG}$ and $\mathtt{REINFORCE}$}&{Assumption \ref{ass:On-policy-derivative},~$\mathbb{V}{\text{ar}}[g(\tau|\theta)]<+\infty$}&$\mathcal{O}(\epsilon^{-4})$\\
          \hline
         {$\mathtt{TRPO}$~ \citep{shani2020adaptive}}&Assumption \ref{ass:On-policy-derivative}&$\mathcal{O}(\epsilon^{-4})$\\
          \hline
         {$\mathtt{TRPO}$~ \citep{neuri-trpo}}&Assumption \ref{ass:On-policy-derivative}&$\mathcal{O}(\epsilon^{-8})$\\
         \hline
         {$\mathtt{SVRPG}$~ \cite{Matteo2018Stochastic}}&\makecell{Assumption \ref{ass:On-policy-derivative},$\mathbb{V}{\text{ar}}[\rho_t]<+\infty$,$\mathbb{V}{\text{ar}}[g(\tau|\theta)]<+\infty$}&$\mathcal{O}(\epsilon^{-4})$\\
         \hline
        {$\mathtt{SVRPG}$~\cite{xu2019improved}}&{Assumption \ref{ass:On-policy-derivative};~$\mathbb{V}{\text{ar}}[\rho_t]<+\infty$,~$\mathbb{V}{\text{ar}}[g(\tau|\theta)]<+\infty$}&$\mathcal{O}(\epsilon^{-10/3})$\\
          \hline
        \makecell{$\mathtt{HAPG}$~\citep{shen2019hessian}} &Assumption \ref{ass:On-policy-derivative}&$\mathcal{O}(\epsilon^{-3})$\\
         \hline
        {$\mathtt{SRVR}$-$\mathtt{PG}$ \citep{xu2020sample,xu2021sample}}&{Assumption \ref{ass:On-policy-derivative};$\mathbb{V}{\text{ar}}[\rho_t]<+\infty$;$\mathbb{V}{\text{ar}}[g(\tau|\theta)]<+\infty$}&$\mathcal{O}(\epsilon^{-3})$
        \\
         \hline
        {$\mathtt{VR}$-$\mathtt{PGPO}$~ \citep{huang2021bregman}}&{Assumption \ref{ass:On-policy-derivative};$\mathbb{V}{\text{ar}}[\rho_t]<+\infty$;$\mathbb{V}{\text{ar}}[g(\tau|\theta)]<+\infty$}&$\mathcal{O}(\epsilon^{-3})$
        \\
         \hline
        {$\mathtt{VRMPO}$ (Our Work)}&
        Assumption \ref{ass:On-policy-derivative}&$\mathcal{O}(\epsilon^{-3})$\\
                \hline
    \end{tabular}
    }
     \caption{
         Comparison of complexity achieves $\|\nabla J(\theta)\|\leq\epsilon$. If $\psi(\theta)=\dfrac{1}{2}\|\theta\|^2_{2}$, then the result (\ref{cr-vrmpo-bound}) of our $\mathtt{VRMPO}$ is also measured by gradient. Beside, $\rho_{t}\overset{\text{def}}=\prod_{i=0}^{H}\dfrac{\pi_{\theta_0}(a_i|s_i)}{\pi_{\theta_t}(a_i|s_i)}$.
         }  
         \label{comparing-complex}   
        \end{table}

\section{Related Works}

\subsection{Stochastic Variance Reduced Gradient in RL}
To our best knowledge, \cite{du2017stochastic} firstly introduce $\mathtt{SVRG}$~\cite{johnson2013accelerating} to off-policy evaluation. 
\cite{du2017stochastic} transform the empirical policy evaluation problem into a convex-concave saddle-point problem, 
then they solve the problem via $\mathtt{SVRG}$ straightforwardly.
Later, to improve sample efficiency for complex RL, \cite{xu2017stochastic} combine $\mathtt{SVRG}$ with $\mathtt{TRPO}$ \citep{schulman2015trust}.
Similarly, \cite{yuan2019policy} introduce $\mathtt{SARAH}$ \citep{nguyen2017sarah} to $\mathtt{TRPO}$ to improve sample efficiency.
However, the results presented by \cite{xu2017stochastic} and \cite{yuan2019policy} are empirical, which lacks a strong theory analysis. 
\cite{metelli2018policy} present a surrogate objective function with R{\'e}nyi divergence \citep{renyi1961measures} to reduce the variance.
Recently, \cite{Matteo2018Stochastic} propose a stochastic variance reduced version of policy gradient ($\mathtt{SVRPG}$),
and they define the gradient estimator via importance sampling as: for each step $k$,
\begin{flalign}
  \nonumber
    \widetilde{G}_{k-1}
    +\dfrac{1}{N}\sum_{j=1}^{N}\left(-g(\tau^{k}_j|\theta_{t})+{\prod_{i=0}^{H}\dfrac{\pi_{\theta_0}(a_i|s_i)}{\pi_{\theta_t}(a_i|s_i)}}g(\tau^{k}_j|\theta_{t-1})\right),
\end{flalign}
where $\widetilde{G}_{k-1}$ is an unbiased estimator according to the trajectory generated by $\pi_{\theta_{k-1}}$.
Although $\mathtt{SVRPG}$ is practical empirically, its gradient estimate is dependent heavily on 
importance sampling.
This fact partially reduces the effectiveness of variance reduction.
Later, \cite{shen2019hessian} remove the importance sampling term, and they construct a Hessian aided policy gradient.

Our $\mathtt{VRMPO}$ is different from  \cite{du2017stochastic}; \cite{xu2017stochastic}; \cite{Matteo2018Stochastic}, which admits a stochastic recursive iteration to estimate the policy gradient.
$\mathtt{VRMPO}$ exploits fresh information to improve convergence and reduces variance. 
Besides, $\mathtt{VRMPO}$ reduces the storage cost since it doesn't require to store the complete historical information.

\subsection{Baseline Methods}

\emph{Baseline} (also also known as control variates) is a widely used technique to reduce the variance~\citep{weaver2001optimal,greensmith2004variance}.
For example, $\mathtt{A2C}$ \citep{sutton1998reinforcement,mnih2016asynchronous} introduces the value function as baseline function, 
\cite{wu2018variance} consider action-dependent baseline, and \cite{liu2018action} use the Stein's identity \citep{stein1986approximate} as baseline.
$\mathtt{Q\text{-}Prop}$ \citep{gu2017q} makes use of both the linear dependent baseline and $\mathtt{GAE}$ \cite{schulman2016high} to reduce variance.
\cite{cheng1019pre} present a predictor-corrector framework transforms a first-order model-free algorithm into a new hybrid method that leverages predictive models to accelerate policy learning.
Mao et al. \cite{mao2019variance} derive a bias-free, input-dependent baseline to reduce variance, and analytically show its benefits over state-dependent baselines. 
Recently, \cite{grathwohl2018backpropagation}; \cite{cheng2019trajectory} provide a standard explanation for the benefits of such approaches with baseline function.

However, the capacity of all the above methods is limited by their choice of baseline function \citep{liu2018action}.
In practice, it is troublesome to design a proper baseline function to reduce the variance of policy gradient estimate.
Our $\mathtt{VRMPO}$ avoids the selection of baseline function, and it uses the current trajectories to construct a novel, efficiently computable gradient to reduce variance and improve sample efficiency.

\section{Experiments}

Our experiments cover the following three different aspects:
\begin{compactenum}[\textbullet]
\item We provide a numerical analysis of MPO, and  compare the convergence rate of $\mathtt{MPO}$ with $\mathtt{REINFORCE}$ and $\mathtt{VPG}$ on the \emph{Short Corridor with Switched Actions} (SASC) domain \citep{sutton2018reinforcement}. 
\end{compactenum}
\begin{compactenum}[\textbullet]
\item We provider a better understand the effect of how the mirror map affects the performance of $\mathtt{VRMPO}$.
\end{compactenum}
\begin{compactenum}[\textbullet]
\item To demonstrate the stability and efficiency of $\mathtt{VRMPO}$ on the MuJoCo continuous control tasks,
we provide a comprehensive comparison to state-of-the-art policy optimization algorithms. 
\end{compactenum}

\subsection{Numerical Analysis of MPO} 
SASC Domain (see Appendix B): The task is to estimate the optimal value function 
of state $\mathtt{s}_{1}$, $V(\mathtt{s}_{1})=G_{0}\approx -11.6$.
Let $\phi(s, \mathtt{right}) = [1, 0]^{\top}$ and $\phi(s, \mathtt{left}) = [0, 1]^{\top}$, $s\in\mathcal{S}$. 
Let $L_{\theta}(s,a)=\phi^{\top}(s,a)\theta$, $(s,a)\in\mathcal{S}\times\mathcal{A}$, where $\mathcal{A}=\{\mathtt{right},\mathtt{left}\}$.
$\pi_{\theta}(a|s)$ is the soft-max distribution defined as 
\[
\pi_{\theta}(a|s)=\dfrac{\exp\{L_{\theta}(s,a)\}}{\sum_{a^{'}\in\mathcal{A}}\exp\{L_{\theta}(s,a^{'})\}}.
\]
The initial parameter $\theta_{0}\sim\mathcal{U}[-0.5,0.5]$, where $\mathcal{U}$ is the uniform distribution.

Before we report the results, it is necessary to explain why we only compare $\mathtt{MPO}$ with $\mathtt{VPG}$ and $\mathtt{REINFORCE}$.
$\mathtt{VPG}$/$\mathtt{REINFORCE}$ is one of the most fundamental policy gradient methods in RL, and extensive modern policy-based algorithms are derived from $\mathtt{VPG}$/$\mathtt{REINFORCE}$.
Our $\mathtt{MPO}$ is a new policy gradient algorithm to learn the parameter.
Thus, it is natural to compare with $\mathtt{VPG}$ and $\mathtt{REINFORCE}$.
The result of Figure \ref{MPO-VPG-RE} shows that $\mathtt{MPO}$ converges faster significantly than both $\mathtt{REINFORCE}$ and $\mathtt{VPG}$.

\subsection{Effect of Mirror Map on VRMPO}
\label{sec: ex-map}

If $\psi(\cdot)$ is $\ell_{p}$-norm, then $\psi^{\star}(y)=(\sum_{i=1}^{n}|y_{i}|^{q})^{\frac{1}{q}}$ is the conjugate map of $\psi$, where $y=(y_1,y_2,\cdots,y_n)^{\top}$, $\dfrac{1}{p}+\dfrac{1}{q}=1$, and $p,q>1$.
According to \citet{beck2003mirror}, iteration (\ref{theta-mirror}) is equivalent to
\[
    \theta_{k+1}=\nabla\psi^{\star}(\nabla\psi(\theta_{k})+\alpha G_k),
\]
where $\nabla\psi_{j}(x)$ and $\nabla\psi_{j}^{\star}(y)$ are:
\[
 \nabla\psi_{j}(x)=\dfrac{\text{sign}(x_{j})|x_j|^{p-1}}{\|x\|_{p}^{p-2}},
  \nabla\psi_{j}^{\star}(y)=\dfrac{\text{sign}(y_{j})|y_j|^{q-1}}{\|y\|_{q}^{q-2}},
\]
 and $j$ is coordinate index of the vector $\nabla\psi$, $\nabla\psi^{\star}$.
 \begin{figure}[h]
    \centering
    \centering
  \includegraphics[scale=0.3]{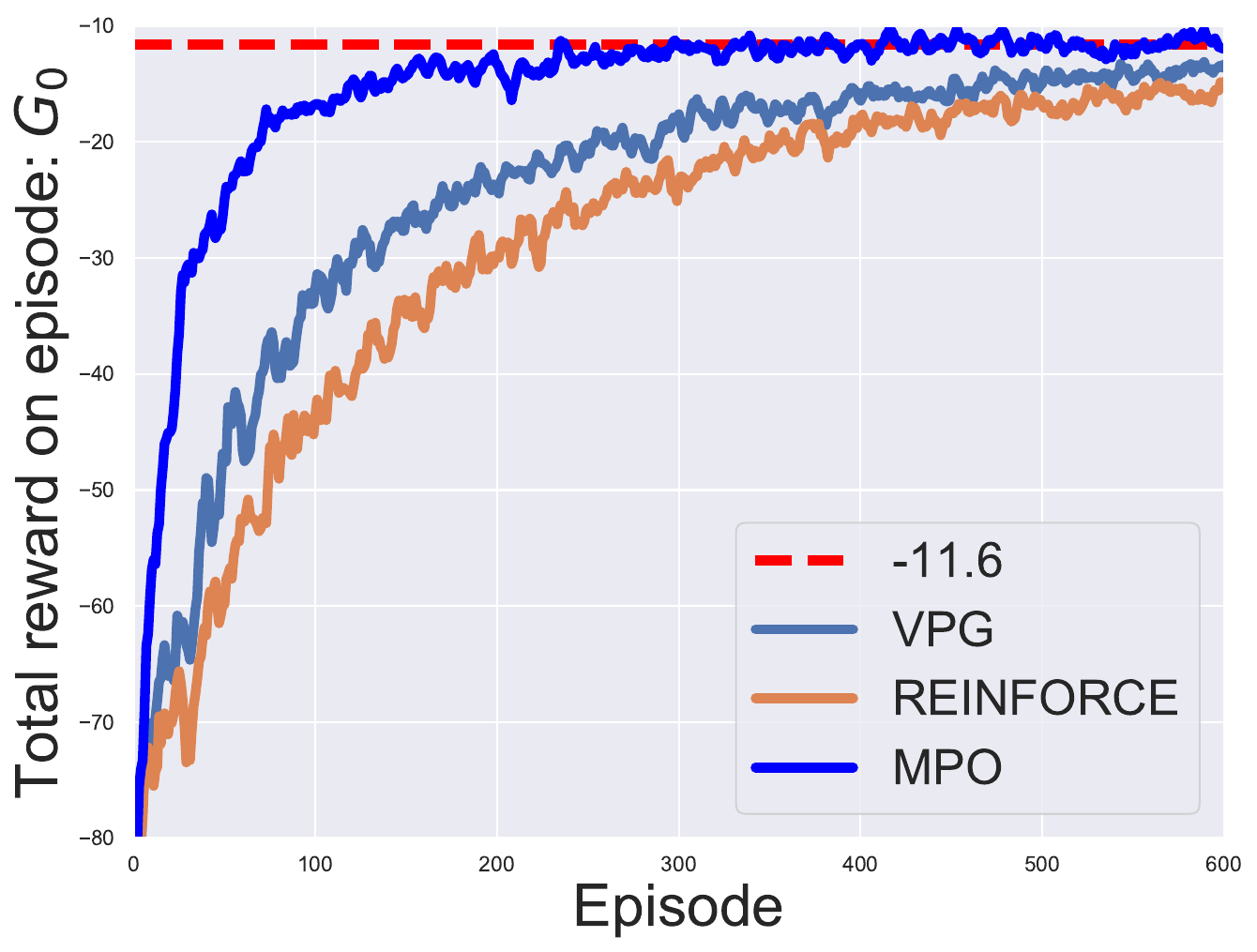}
    \caption
    {
        Convergence comparison between our $\mathtt{MPO}$ algorithm and $\mathtt{REINFORCE}$/$\mathtt{VPG}$ on the SASC domain.
    }
    \label{MPO-VPG-RE}
    \end{figure}

To compare fairly, we use the same random seed for each domain. The hyper-parameter $p$ runs in the set $[P]=\{1.1,1.2,\cdots,1.9,2,3,4,5\}$.
For the non-Euclidean distance case, we only show the results of $p=3,4,5$ in Figure \ref{comparing-3-domain}, 
and ``best" is a certain hyper-parameter $p\in[P]$ achieves the best performance among the set $[P]$.
We use a two-layer feedforward neural network of 200 and 100 hidden nodes, respectively, with rectified linear units (ReLU) activation function between each layer.
We run the discounter $ \gamma= 0.99$ and the step-size $\alpha$ is chosen by a grid search from the set $\{0.01,0.02,0.04,0.08, 0.1\}$.

The result of Figure \ref{comparing-3-domain} shows that the best method is produced by non-Euclidean distance ($p\ne 2$), not the Euclidean distance ($p= 2$). 
The traditional policy gradient methods such as $\mathtt{REINFORCE}$, $\mathtt{VPG}$, and $\mathtt{DPG}$ are all the algorithms update parameters by Euclidean distance.
This experiment gives us some light that one can create better algorithms with existing approaches via non-Euclidean distance.
Additionally, the result of Figure \ref{comparing-3-domain} shows our $\mathtt{VRMPO}$ converges faster than $\mathtt{REINFORCE}$, i.e., $\mathtt{VRMPO}$ 
needs less sampled trajectories to reach a convergent state, which supports the complexity analysis in Table \ref{comparing-complex}.
Although $\mathtt{SRVR}$-$\mathtt{PG}$ achieves the same sample complexity as our $\mathtt{VRMPO}$, result of Figure \ref{comparing-3-domain} shows $\mathtt{VRMPO}$ converges faster than $\mathtt{SRVR}$-$\mathtt{PG}$.

\begin{figure*}[t]
    \centering
    {\includegraphics[width=17.3cm,]{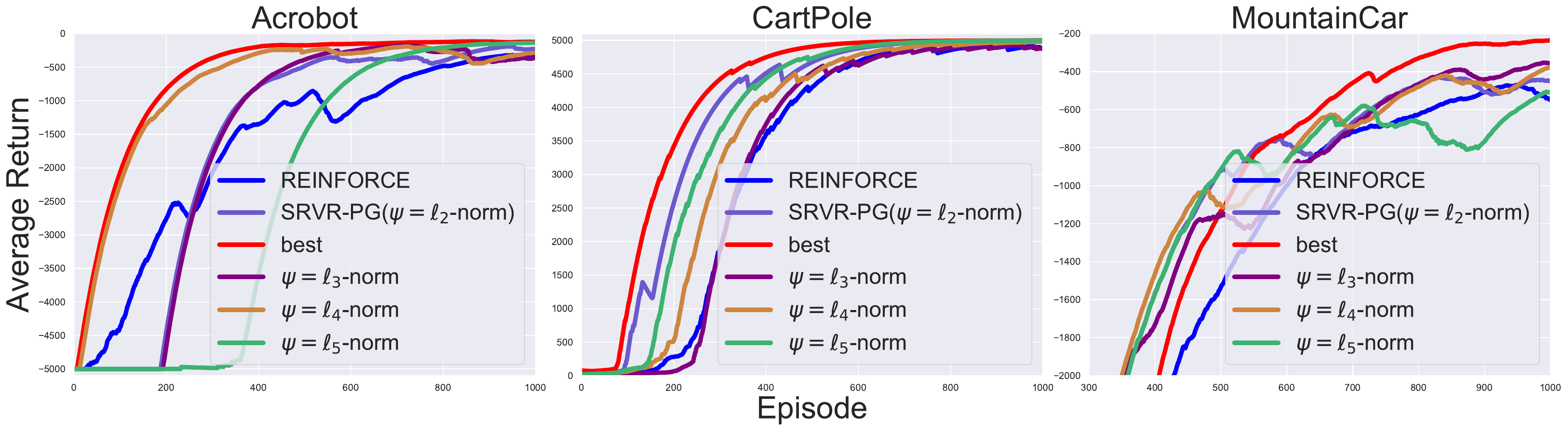}}
    \caption
    {
        Comparison of the empirical performance of $\mathtt{VRMPO}$ between different mirror maps and $\mathtt{REINFORCE}$.
    }
    \label{comparing-3-domain}
\end{figure*}

\subsection{Evaluate VRMPO on Continuous Control Tasks}
\label{on-line-vrmpo-section}


It is noteworthy that the policy gradient (\ref{Eq:G_t}) of $\mathtt{VRMPO}$ is an off-line estimator likes $\mathtt{REINFOECE}$.
As pointed by \cite{sutton2018reinforcement}, $\mathtt{REINFOECE}$ converge asymptotically to a local minimum, but like all off-line methods, it is inconvenient for continuous control tasks, and it is limited in the application to some complex domains. 
This could also happen in $\mathtt{VRMPO}$.

Now, we introduce some practical tricks for on-line implementation of $\mathtt{VRMPO}$. We have provided the complete update rule of on-line $\mathtt{VRMPO}$ in Algorithm \ref{on-line VRMPO}.

\textbf{Details of Implementation}. Firstly, we extend Algorithm \ref{alg:SVR-MPD} to be an actor-critic structure, i.e., we introduce a critic structure to Algorithm \ref{alg:SVR-MPD}.
Concretely, for each step $t$, we construct a critic network $Q_{\omega}(s,a)$ with the parameter $\omega$,
sample $\{(s_i,a_i)\}_{i=1}^{N}$ from a data memory $\mathcal{D}$,
and learn the parameter $\omega$ via minimizing the critic loss as follows,
\begin{flalign}
L_{\omega}=\dfrac{1}{N}\sum_{i=1}^{N}(r_{i+1}+\gamma Q_{\omega_{k-1}}(s_i,a_i)-Q_{\omega}(s_i,a_i))^{2}.
\end{flalign}
For more details, please see $\mathtt{Line~17}$-$\mathtt{20}$ of Algorithm \ref{on-line VRMPO}.
Then, for each pair $(s,a)\sim\mathcal{D}$, we conduct the actor loss \[L_{\theta}(s,a)=-\log\pi_{\theta}(s,a)Q_{\omega_{k-1}}(s,a)\] to replace $J(\theta)$ 
to learn parameter $\theta$.
For more details, please see $\mathtt{Line~9}$-$\mathtt{16}$ of Algorithm \ref{on-line VRMPO} (Appendix E.1).

\textbf{Score Performance Comparison}. 

From the results of Figure \ref{fig-learning-curves} and Table \ref{table-max-data}, overall, $\mathtt{VRMPO}$ outperforms the baseline algorithms in both final performance and learning process. 
Our $\mathtt{VRMPO}$ also learns considerably faster with better performance than the popular $\mathtt{TD3}$ on Walker2d, HalfCheetah, Hopper, InvDoublePendulum (IDP), and Reacher domains.
On the InvDoublePendulum task, our $\mathtt{VRMPO}$ has only a small advantage over other algorithms. 
This is because the InvPendulum task is relatively easy.
The advantage of our $\mathtt{VRMPO}$ becomes more powerful when the task is more difficult. 
It is worth noticing that on the HalfCheetah domain, our $\mathtt{VRMPO}$ achieves a significant max-average score 16000+, which outperforms far more than the second-best score 11781.

\begin{figure*}[t!]
      \label{learning-curves}
    \centering
    \subfigure[Walker2d-v2]
    {\includegraphics[width=5cm,height=4.2cm]{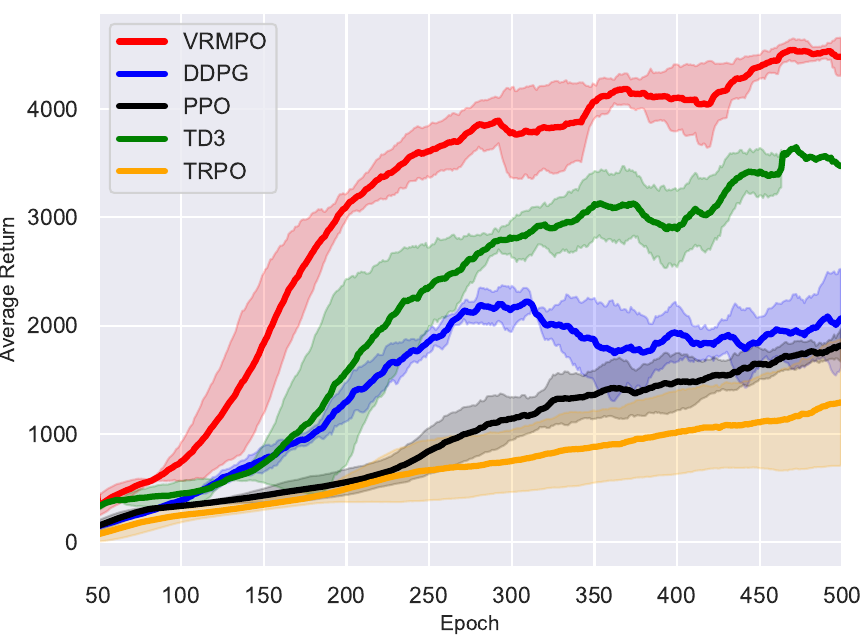}}
    \subfigure[HalfCheetah-v2]
    {\includegraphics[width=5cm,height=4.2cm]{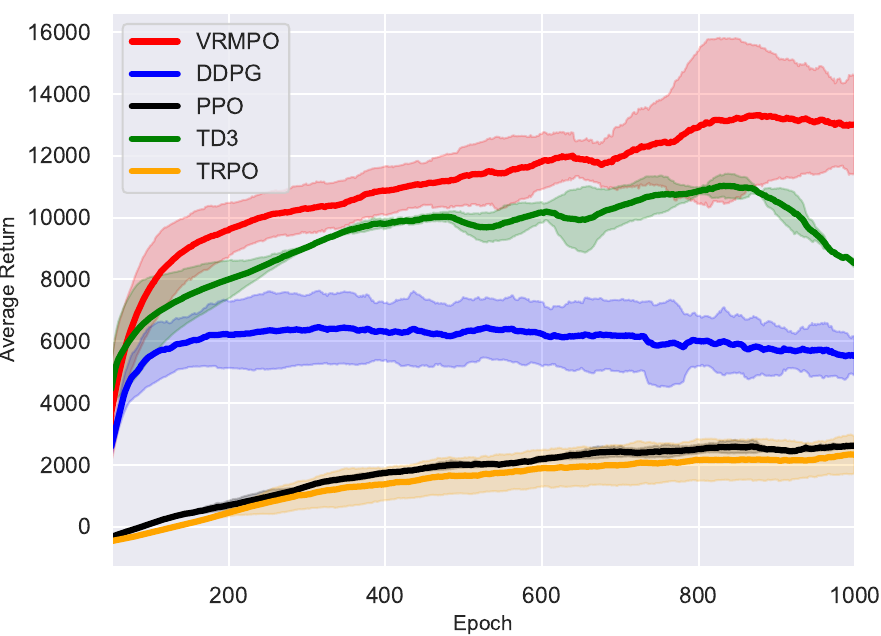}}
    \subfigure[Reacher-v2]
    {\includegraphics[width=5cm,height=4.2cm]{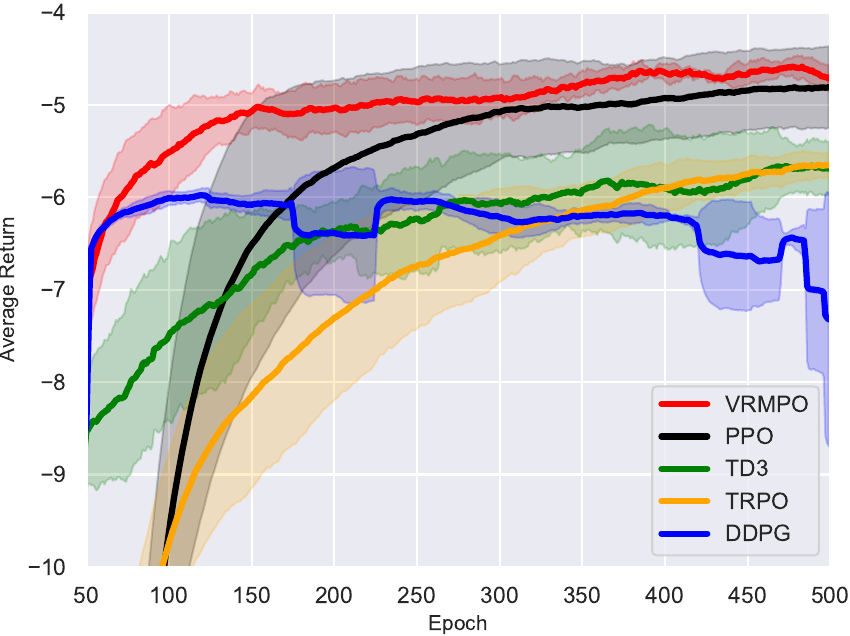}}
    \subfigure[Hopper-v2]
    {\includegraphics[width=5cm,height=4.2cm]{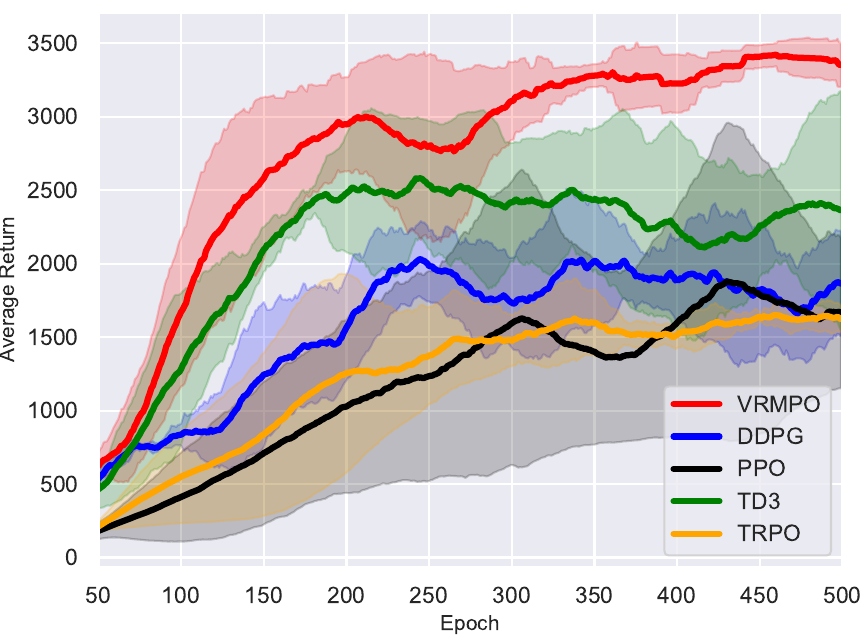}}
    \subfigure[InvDoublePendulum-v2]
    {\includegraphics[width=5cm,height=4.2cm]{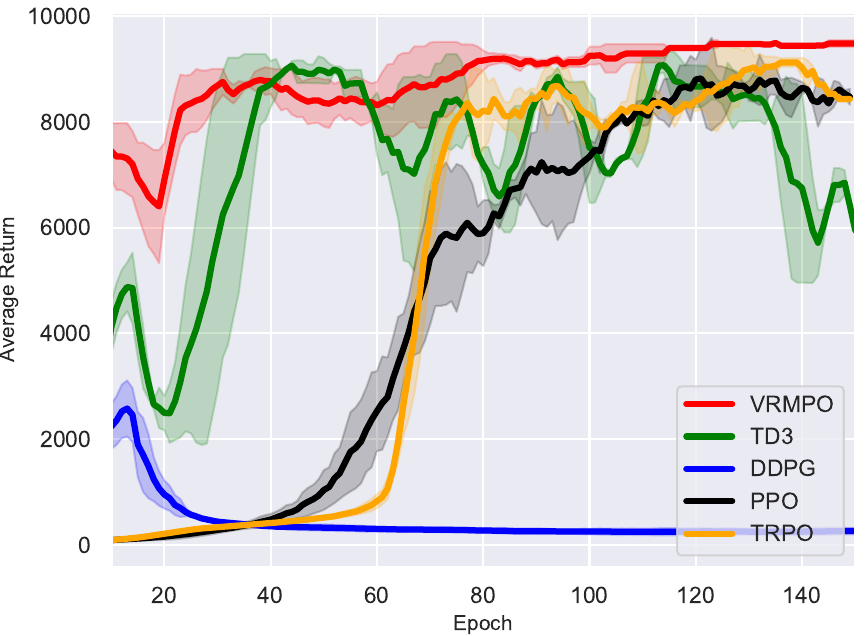}}
    \subfigure[InvPendulum-v2]
    {\includegraphics[width=5cm,height=4.2cm]{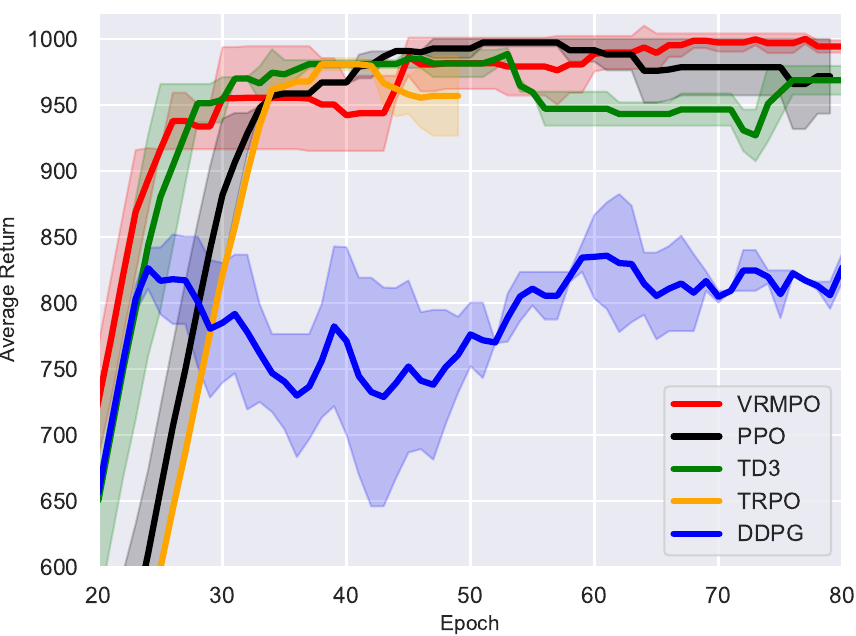}}
    \caption
    {
        Learning curves for continuous control tasks. The shaded region represents the standard deviation of the score over the best three trials. Curves are smoothed uniformly for visual clarity.
    }
    \label{fig-learning-curves}
\end{figure*}
\begin{table}[t!]
\centering
\label{results}
\begin{center}
\footnotesize{
\begin{tabular}{p{3cm}p{2cm}p{2cm}p{2cm}p{2cm}p{2cm}}
\toprule
\bf{Environment} & \bf{VRMPO} & \bf{TD3} & \bf{DDPG} & \bf{PPO} & \bf{TRPO}  \\
\midrule
Walker2d    & 5251.83 & 4887.85    &  \bf{5795.13}     & 3905.99        &3636.59  \\  
HalfCheetah             & \bf{16095.51} & 11781.07        & 8616.29    & 3542.60    & 3325.23  \\
Reacher         & -0.49 & -1.47        & -1.55         & \bf{-0.44}    & -0.66  \\
Hopper            & \bf{3751.43} & 3482.06         & 3558.69        & 3609.65        & 3578.06   \\
InvPendulum    & \bf{9359.82}     & 9248.27    & 6958.42         &9045.86    & 9151.56 \\
InvPendulum    & \bf{1000.00}         & \bf{1000.00}     & 907.81  & \bf{1000.00}    & \bf{1000.00}   \\
\bottomrule
\end{tabular}
}
\caption{Max-average return over final 50 epochs, where we run 5000 iterations for each epoch.
Maximum value for each task is bolded.}
\label{table-max-data}
\end{center}
\vskip -0.1in
\end{table}

\textbf{Stability}. 

The stability of an algorithm is also an important topic in RL.
Although $\mathtt{DDPG}$ exploits the off-policy samples, which promotes its efficiency in stable environments.
$\mathtt{DDPG}$ is unstable on the Reacher task, while our $\mathtt{VRMPO}$ learning faster significantly with lower variance.
$\mathtt{DDPG}$ fails to make any progress on InvDoublePendulum domain, which is corroborated by \cite{dai2017sbeed}.
Although $\mathtt{TD3}$ takes the minimum value between a pair of critics to limit overestimation, it learns severely fluctuating in the InvertedDoublePendulum environment.
In contrast, our $\mathtt{VRMPO}$ is consistently reliable and effective in different tasks.

\textbf{Variance Comparison}. 

As we can see from the results in Figure \ref{fig-learning-curves}, our $\mathtt{VRMPO}$ converges with a considerably low variance in the Hopper, InvDoublePendulum, and Reacher.
Although the asymptotic variance of $\mathtt{VRMPO}$ is slightly larger than other algorithms in HalfCheetah, the final performance of $\mathtt{VRMPO}$ outperforms all the baselines significantly.
The result of Figure \ref{fig-learning-curves} also implies conducting a proper gradient estimator not only reduces the variance of the score during the learning but speeds the convergence of training.

\section{Conclusion}

In this paper, we analyze the theoretical dilemma of applying SMD to policy optimization.
Then, we  propose a sample efficient algorithm $\mathtt{VRMPO}$, and prove the sample complexity of $\mathtt{VRMPO}$ achieves only $\mathcal{O}(\epsilon^{-3})$.
To our best knowledge, $\mathtt{VRMPO}$ matches the best sample complexity so far.
Finally, we conduct extensive experiments to show our algorithm outperforms state-of-the-art policy gradient methods.

\clearpage

\bibliographystyle{apalike}
\bibliography{reference}

\begin{thebibliography}{}

\bibitem[Agarwal et~al., 2020]{agarwal2020optimality}
Agarwal, A., Kakade, S.~M., Lee, J.~D., and Mahajan, G. (2020).
\newblock Optimality and approximation with policy gradient methods in markov
  decision processes.
\newblock {\em COLT}.

\bibitem[Bauschke et~al., 2011]{bauschke2011convex}
Bauschke, H.~H., Combettes, P.~L., et~al. (2011).
\newblock {\em Convex analysis and monotone operator theory in Hilbert spaces},
  volume 408.

\bibitem[Beck and Teboulle, 2003]{beck2003mirror}
Beck, A. and Teboulle, M. (2003).
\newblock Mirror descent and nonlinear projected subgradient methods for convex
  optimization.
\newblock {\em Operations Research Letters}, 31(3):167--175.

\bibitem[Bertsekas, 2009]{bertsekas2009convex}
Bertsekas, D.~P. (2009).
\newblock {\em Convex optimization theory}.
\newblock Athena Scientific Belmont.

\bibitem[Cheng et~al., 2019a]{cheng2019trajectory}
Cheng, C.-A., Yan, X., and Boots, B. (2019a).
\newblock Trajectory-wise control variates for variance reduction in policy
  gradient methods.
\newblock {\em ICRA}.

\bibitem[Cheng et~al., 2019b]{cheng1019pre}
Cheng, C.-A., Yan, X., Ratliff, N., and Boots, B. (2019b).
\newblock Predictor-corrector policy optimization.
\newblock In {\em ICML}.

\bibitem[Dai et~al., 2018]{dai2017sbeed}
Dai, B., Shaw, A., Li, L., Xiao, L., He, N., Liu, Z., Chen, J., and Song, L.
  (2018).
\newblock Sbeed: Convergent reinforcement learning with nonlinear function
  approximation.
\newblock {\em ICML}.

\bibitem[Du et~al., 2017]{du2017stochastic}
Du, S.~S., Chen, J., Li, L., Xiao, L., and Zhou, D. (2017).
\newblock Stochastic variance reduction methods for policy evaluation.
\newblock In {\em ICML}.

\bibitem[Duan et~al., 2016]{duan2016benchmarking}
Duan, Y., Chen, X., Houthooft, R., Schulman, J., and Abbeel, P. (2016).
\newblock Benchmarking deep reinforcement learning for continuous control.
\newblock In {\em ICML}.

\bibitem[Fujimoto et~al., 2018]{fujimoto2018addressing}
Fujimoto, S., van Hoof, H., Meger, D., et~al. (2018).
\newblock Addressing function approximation error in actor-critic methods.
\newblock {\em ICML}.

\bibitem[Ghadimi et~al., 2016]{ghadimi2016mini}
Ghadimi, S., Lan, G., and Zhang, H. (2016).
\newblock Mini-batch stochastic approximation methods for nonconvex stochastic
  composite optimization.
\newblock {\em Mathematical Programming}, 155(1-2):267--305.

\bibitem[Grathwohl et~al., 2018]{grathwohl2018backpropagation}
Grathwohl, W., Choi, D., Wu, Y., Roeder, G., and Duvenaud, D. (2018).
\newblock Backpropagation through the void: Optimizing control variates for
  black-box gradient estimation.
\newblock {\em ICLR}.

\bibitem[Greensmith et~al., 2004]{greensmith2004variance}
Greensmith, E., Bartlett, P.~L., Baxter, J., et~al. (2004).
\newblock Variance reduction techniques for gradient estimates in reinforcement
  learning.
\newblock {\em JMLR}, 5(Nov):1471--1530.

\bibitem[Gu et~al., 2017]{gu2017q}
Gu, S., Lillicrap, T., Ghahramani, Z., Turner, R.~E., and Levine, S. (2017).
\newblock Q-prop: Sample-efficient policy gradient with an off-policy critic.
\newblock {\em ICLR}.

\bibitem[Haarnoja et~al., 2018]{haarnoja2018soft}
Haarnoja, T., Zhou, A., Abbeel, P., and Levine, S. (2018).
\newblock Soft actor-critic: Off-policy maximum entropy deep reinforcement
  learning with a stochastic actor.
\newblock In {\em ICML}.

\bibitem[Huang et~al., 2021]{huang2021bregman}
Huang, F., Gao, S., Huang, H., and et.al (2021).
\newblock Bregman gradient policy optimization.
\newblock {\em arXiv preprint arXiv:2106.12112}.

\bibitem[Jain et~al., 2017]{jain2017non}
Jain, P., Kar, P., et~al. (2017).
\newblock Non-convex optimization for machine learning.
\newblock {\em Foundations and Trends{\textregistered} in Machine Learning},
  10(3-4):142--336.

\bibitem[Johnson and Zhang, 2013]{johnson2013accelerating}
Johnson, R. and Zhang, T. (2013).
\newblock Accelerating stochastic gradient descent using predictive variance
  reduction.
\newblock In {\em NeurIPS}, pages 315--323.

\bibitem[Kakade, 2002]{kakade2002natural}
Kakade, S.~M. (2002).
\newblock A natural policy gradient.
\newblock In {\em NeurIPS}, pages 1531--1538.

\bibitem[Kingma and Ba, 2015]{kingma2015adam}
Kingma, D.~P. and Ba, J. (2015).
\newblock Adam: A method for stochastic optimization.
\newblock {\em ICLR}.

\bibitem[Konda and Tsitsiklis, 2000]{konda2000actor}
Konda, V.~R. and Tsitsiklis, J.~N. (2000).
\newblock Actor-critic algorithms.
\newblock In {\em NeurIPS}, pages 1008--1014.

\bibitem[Lei and Tang, 2018]{lei2018stochastic}
Lei, Y. and Tang, K. (2018).
\newblock Stochastic composite mirror descent: optimal bounds with high
  probabilities.
\newblock In {\em NeurIPS}, pages 1519--1529.

\bibitem[Liu et~al., 2019]{neuri-trpo}
Liu, B., Cai, Q., Yang, Z., and Wang, Z. (2019).
\newblock Neural proximal/trust region policy optimization attains globally
  optimal policy.
\newblock {\em NeurIPS}.

\bibitem[Liu et~al., 2018]{liu2018action}
Liu, H., Feng, Y., Mao, Y., Zhou, D., Peng, J., and Liu, Q. (2018).
\newblock Action-dependent control variates for policy optimization via stein
  identity.
\newblock {\em ICLR}.

\bibitem[Mao et~al., 2019]{mao2019variance}
Mao, H., Venkatakrishnan, S.~B., Schwarzkopf, M., and Alizadeh, M. (2019).
\newblock Variance reduction for reinforcement learning in input-driven
  environments.
\newblock {\em ICLR}.

\bibitem[Meng et~al., 2019]{meng2019qualitative}
Meng, W., Zheng, Q., Yang, L., Li, P., and Pan, G. (2019).
\newblock Qualitative measurements of policy discrepancy for return-based deep
  q-network.
\newblock {\em IEEE transactions on neural networks and learning systems},
  31(10):4374--4380.

\bibitem[Metelli et~al., 2018]{metelli2018policy}
Metelli, A.~M., Papini, M., Faccio, F., and Restelli, M. (2018).
\newblock Policy optimization via importance sampling.
\newblock In {\em NeurIPS}, pages 5442--5454.

\bibitem[Mnih et~al., 2016]{mnih2016asynchronous}
Mnih, V., Badia, A.~P., Mirza, M., Graves, A., Lillicrap, T., Harley, T.,
  Silver, D., and Kavukcuoglu, K. (2016).
\newblock Asynchronous methods for deep reinforcement learning.
\newblock In {\em ICML}, pages 1928--1937.

\bibitem[Nemirovsky and Yudin, 1983]{nemirovsky1983problem}
Nemirovsky, A.~S. and Yudin, D.~B. (1983).
\newblock Problem complexity and method efficiency in optimization.

\bibitem[Nguyen et~al., 017a]{nguyen2017sarah}
Nguyen, L.~M., Liu, J., Scheinberg, K., and Tak{\'a}{\v{c}}, M. (2017a).
\newblock Sarah: A novel method for machine learning problems using stochastic
  recursive gradient.
\newblock In {\em ICML}.

\bibitem[Papini et~al., 2018]{Matteo2018Stochastic}
Papini, M., Binaghi, D., Canonaco, G., and Matteo~Pirotta, M.~R. (2018).
\newblock Stochastic variance-reduced policy gradient.
\newblock In {\em ICML}.

\bibitem[Peters et~al., 2010]{peters2010relative}
Peters, J., M{\"u}lling, K., Altun, and Yasemin (2010).
\newblock Relative entropy policy search.
\newblock In {\em AAAI}, pages 1607--1612.

\bibitem[Reddi et~al., 2016]{reddi2016stochastic}
Reddi, S.~J., Hefny, A., Sra, S., Poczos, B., and Smola, A. (2016).
\newblock Stochastic variance reduction for nonconvex optimization.
\newblock In {\em ICML}, pages 314--323.

\bibitem[R{\'e}nyi et~al., 1961]{renyi1961measures}
R{\'e}nyi, A. et~al. (1961).
\newblock On measures of entropy and information.
\newblock In {\em Proceedings of the Fourth Berkeley Symposium on Mathematical
  Statistics and Probability, Volume 1: Contributions to the Theory of
  Statistics}.

\bibitem[Schulman et~al., 2015]{schulman2015trust}
Schulman, J., Levine, S., Abbeel, P., Jordan, M., and Moritz, P. (2015).
\newblock Trust region policy optimization.
\newblock In {\em ICML}, pages 1889--1897.

\bibitem[Schulman et~al., 2016]{schulman2016high}
Schulman, J., Moritz, P., Levine, S., Jordan, M., and Abbeel, P. (2016).
\newblock High-dimensional continuous control using generalized advantage
  estimation.
\newblock {\em ICLR}.

\bibitem[Shani et~al., 2020]{shani2020adaptive}
Shani, L., Efroni, Y., Mannor, S., and et.al (2020).
\newblock Adaptive trust region policy optimization: Global convergence and
  faster rates for regularized mdps.
\newblock In {\em AAAI}, volume~34, pages 5668--5675.

\bibitem[Shen et~al., 2019]{shen2019hessian}
Shen, Z., Ribeiro, A., Hassani, H., Qian, H., and Mi, C. (2019).
\newblock Hessian aided policy gradient.
\newblock In {\em ICML}, pages 5729--5738.

\bibitem[Silver et~al., 2016]{silver2016mastering}
Silver, D., Huang, A., Maddison, C.~J., Guez, A., Sifre, L., Van Den~Driessche,
  G., Schrittwieser, J., Antonoglou, I., Panneershelvam, V., Lanctot, M.,
  et~al. (2016).
\newblock Mastering the game of go with deep neural networks and tree search.
\newblock {\em nature}, 529(7587):484--489.

\bibitem[Silver et~al., 2014]{silver2014deterministic}
Silver, D., Lever, G., Heess, N., Degris, T., Wierstra, D., and Riedmiller, M.
  (2014).
\newblock Deterministic policy gradient algorithms.
\newblock In {\em ICML}.

\bibitem[Silver et~al., 2017]{silver2017mastering}
Silver, D., Schrittwieser, J., Simonyan, K., Antonoglou, I., Huang, A., Guez,
  A., Hubert, T., Baker, L., Lai, M., Bolton, A., et~al. (2017).
\newblock Mastering the game of go without human knowledge.
\newblock {\em Nature}, 550(7676):354.

\bibitem[Stein, 1986]{stein1986approximate}
Stein, C. (1986).
\newblock Approximate computation of expectations.
\newblock {\em Lecture Notes-Monograph Series}, 7:i--164.

\bibitem[Sutton and Barto, 1998]{sutton1998reinforcement}
Sutton, R.~S. and Barto, A.~G. (1998).
\newblock {\em Reinforcement learning: An introduction}.
\newblock MIT press.

\bibitem[Sutton and Barto, 2018]{sutton2018reinforcement}
Sutton, R.~S. and Barto, A.~G. (2018).
\newblock {\em Reinforcement learning: An introduction}.
\newblock MIT press.

\bibitem[Sutton et~al., 2000]{sutton2000policy}
Sutton, R.~S., McAllester, D.~A., Singh, S.~P., and Mansour, Y. (2000).
\newblock Policy gradient methods for reinforcement learning with function
  approximation.
\newblock In {\em NeurIPS}, pages 1057--1063.

\bibitem[Tomar et~al., 2020]{tomar2020mirror}
Tomar, M., Shani, L., Efroni, Y., and Ghavamzadeh, M. (2020).
\newblock Mirror descent policy optimization.
\newblock {\em arXiv preprint arXiv:2005.09814}.

\bibitem[Van~Hasselt et~al., 2016]{van2016deep}
Van~Hasselt, H., Guez, A., and Silver, D. (2016).
\newblock Deep reinforcement learning with double q-learning.
\newblock In {\em AAAI}.

\bibitem[Weaver and Tao, 2001]{weaver2001optimal}
Weaver, L. and Tao, N. (2001).
\newblock The optimal reward baseline for gradient-based reinforcement
  learning.
\newblock In {\em UAI}, pages 538--545. Morgan Kaufmann Publishers Inc.

\bibitem[Williams, 1992]{williams1992simple}
Williams, R.~J. (1992).
\newblock Simple statistical gradient-following algorithms for connectionist
  reinforcement learning.
\newblock {\em Machine learning}, 8(3-4):229--256.

\bibitem[Wu et~al., 2018]{wu2018variance}
Wu, C., Rajeswaran, A., Duan, Y., Kumar, V., Bayen, A.~M., Kakade, S.,
  Mordatch, I., and Abbeel, P. (2018).
\newblock Variance reduction for policy gradient with action-dependent
  factorized baselines.
\newblock {\em ICLR}.

\bibitem[Xing et~al., 2021]{xinglearning2021}
Xing, D., Liu, Q., Zheng, Q., and Pan, G. (2021).
\newblock Learning with generated teammates to achieve type-free ad-hoc
  teamwork.
\newblock In {\em IJCAI}.

\bibitem[Xu, 2021]{xu2021sample}
Xu, P. (2021).
\newblock {\em Sample-Efficient Nonconvex Optimization Algorithms in Machine
  Learning and Reinforcement Learning}.
\newblock PhD thesis, UCLA.

\bibitem[Xu et~al., 2019]{xu2019improved}
Xu, P., Gao, F., Gu, Q., et~al. (2019).
\newblock An improved convergence analysis of stochastic variance-reduced
  policy gradient.
\newblock {\em UAI}.

\bibitem[Xu et~al., 2020]{xu2020sample}
Xu, P., Gao, F., Gu, Q., et~al. (2020).
\newblock Sample efficient policy gradient methods with recursive variance
  reduction.
\newblock {\em ICLR}.

\bibitem[Xu et~al., 2017]{xu2017stochastic}
Xu, T., Liu, Q., and Peng, J. (2017).
\newblock Stochastic variance reduction for policy gradient estimation.
\newblock {\em arXiv preprint arXiv:1710.06034}.

\bibitem[Yang et~al., 2021]{yang2021sample}
Yang, L., Zheng, Q., and Pan, G. (2021).
\newblock Sample complexity of policy gradient finding second-order stationary
  points.
\newblock {\em AAAI2021}.

\bibitem[Yuan et~al., 2019]{yuan2019policy}
Yuan, H., Li, C.~J., Tang, Y., and Zhou, Y. (2019).
\newblock Policy optimization via stochastic recursive gradient algorithm.
\newblock \url{https://openreview.net/forum?id=rJl3S2A9t7}.

\end{thebibliography}

\clearpage
\appendix
\section{Notations}

For convenience of reference, we list some key notations that have be used in this paper.

\begin{tabular}{r c p{15cm}}
        $J(\theta)$ &: &The expected return function.\\
        $\mathcal{J}(\theta)$ &: & $\mathcal{J}(\theta)$=$-J(\theta)$.\\
        $L$ &: &The Lipschiz constant, see (\ref{def:L}). \\
	$\sigma$ &: & The boundedness of the variance of the policy gradient estimator, see (\ref{def:sigma}). \\
	$\psi$&: & The mirror map.\\
	$\zeta$&: & The mirror map $\psi$ is a $\zeta$-strictly convex function, and we present it in Definition \ref{def1}.\\
	$\mathcal{G}(\cdot)$ &: &Bregman gradient, see (\ref{bregman-grdient-mapping}).\\
	$\Delta$&: & $\Delta=J(\theta^{\star})-J(\theta_1)$.
\end{tabular}

\section{Appendix A: Proof of Theorem \ref{cr-theo1}}
\label{sec: app-A}

\textbf{Theorem \ref{cr-theo1}} (Convergence Rate of Algorithm \ref{alg:Mirror-Policy-Algorithm})
\emph{
		Under Assumption \ref{ass:On-policy-derivative}, and the total trajectories are $\{\tau_{k}\}_{k=1}^{N}$.
    Consider the sequence $\{\theta_{k}\}_{k=1}^{N}$ generated by Algorithm \ref{alg:Mirror-Policy-Algorithm}, and the output $\tilde{\theta}_{N}=\theta_{n}$ follows the distribution of Eq.(\ref{mass-distriution}).
    Let $0<\alpha_{k}<\dfrac{\zeta}{L}$,
$\ell(g,u)=\langle g,u\rangle$. Let $\hat{g}_{k}=\dfrac{1}{k}\sum_{i=1}^{k}g_i$, where 
  ${g}_{i}=\sum_{t=0}^{H_{\tau_{i}}}
  \nabla_{\theta}\log\pi_{\theta}(a_{t}|s_{t})R(\tau_{i})|_{\theta=\theta_i}$.
   Then we have
    \begin{flalign}
    \nonumber
         \mathbb{E}[\|\mathcal{G}_{\alpha_n,\ell(-{g}_n,{\theta}_{n})}^{\psi}({\theta}_{n})\|^{2}]\leq\dfrac{\big(J(\theta^{\star})-J(\theta_1)\big){+\dfrac{\sigma^{2}}{\zeta}\sum_{k=1}^{N}{\dfrac{\alpha_k}{k}}}}{{\sum_{k=1}^{N}(\zeta\alpha_{k}-L\alpha^{2}_{k})}}.
    \end{flalign}
}

Let $f(\theta)$ be a $L$-smooth function defined on $\mathbb{R}^{n}$, i.e $\|\nabla f(\theta)-\nabla f(\theta^{'})\|\leq L\|\theta-\theta^{'}\|$. Then, according to \cite{jain2017non}, for $\forall \theta,\theta^{'}\in$$\mathbb{R}^{n}$, the following holds
 \begin{flalign}
 	\label{app:l-mooth-inequality}
 	\| f(\theta)- f(\theta^{'})-\langle\nabla f(\theta^{'}),\theta-\theta^{'}\rangle\|\leq\dfrac{L}{2}\|\theta-\theta^{'}\|^{2}.
 \end{flalign}

\begin{lemma}
	[\cite{ghadimi2016mini}, Lemma 1 and Proposition 1]
	\label{app-a-lemma}
	Let $\mathcal{X}$ be a closed convex set in $\mathbb{R}^d$, 
	$f(\theta)$ is a $L$-smooth function defined on $\mathcal{X}$,
	$h(\cdot):\mathcal{X} \rightarrow \mathbb{R}$ be a convex function, but possibly nonsmooth,  and $D_{\psi}: \mathcal{X}\times \mathcal{X} \rightarrow \mathbb{R}$ is Bregman divergence. Moreover,  we define
	\begin{flalign}
	\nonumber
	&\theta^+ = \arg \min_{u\in \mathcal{X}} \{\langle g,u\rangle + \dfrac{1}{\alpha}D_{\psi}(u,\theta) + h(u) \} \\
	\label{app-A-eq-2}
	&P_{\mathcal{X}}(\theta,g,\alpha) = \dfrac{1}{\alpha}(\theta - \theta^+), 
	\end{flalign}
	where $g =\nabla f(\theta) \in \mathbb{R}^d$, $\theta \in \mathcal{X}$, and  $\alpha > 0$. Then, the following statement holds 
	\begin{flalign}
    \label{app-A-eq-3}
	\langle g,P_{\mathcal{X}} (\theta,g,\alpha)\rangle
		\ge \zeta \|P_{\mathcal{X}}(\theta,g,\alpha)\|^2 + \dfrac{1}{\alpha}[h(\theta^+) - h(\theta)],
	\end{flalign} 	
	where $\zeta$ is a positive constant determined by $\psi$ (i.e. $\psi$ is a 
	a continuously-differentiable and $\zeta$-strictly convex function) that satisfies
	$\langle \theta-\theta^{'},\nabla \psi(\theta)-\nabla \psi(\theta^{'})\rangle\ge\zeta\|\theta-\theta^{'}\|^{2}$.
	Moreover, for any $g_1,g_2 \in \mathbb{R}^d$, the following statement holds 
	\begin{flalign}
		\label{app-A-eq-4}
	\|P_{\mathcal{X}}(\theta,g_1,\alpha) - P_{\mathcal{X}}(\theta,g_2,\alpha)   \| &\leq \dfrac{1}{\zeta} \|g_1 - g_2\|. 
	\end{flalign}
\end{lemma}

\begin{remark}
Ghadimi, Lan, and Zhang \cite{ghadimi2016mini} call $P_{\mathcal{X}}(\theta,g,\alpha)$ (\ref{app-A-eq-2}) \emph{projected gradient}. It is noteworthy that if $h(\cdot)=0$, then the projected gradient $P_{\mathcal{X}}(x,g,\alpha)$ is reduced to Bregman Gradient (\ref{bregman-grdient-mapping}).
Concretely,
for Eq.(\ref{app-A-eq-2}), let $h(\cdot)\equiv 0$, then we have
\[P_{\mathcal{X}}(\theta,g,\alpha) \overset{(\ref{app-A-eq-2})}=  \dfrac{1}{\alpha}\big(\theta - \arg \min_{u\in \mathcal{X}} \{\langle g,u\rangle + \dfrac{1}{\alpha}D_{\psi}(u,\theta) \} \big)\overset{(\ref{bregman-grdient-mapping})}=\mathcal{G}_{\alpha,\langle g,\theta\rangle}^{\psi}(\theta).\]
\end{remark}

\begin{proof} \textbf{(Proof of Theorem \ref{cr-theo1})}

	Let $\mathcal{T}=\{\tau_{k}\}_{k=1}^{N}$,
	at each terminal end of a trajectory $\tau_{k}=\{s_{t},a_{t},r_{t+1}\}_{t=0}^{H_{\tau_{k}}}\in\mathcal{T}$, let 
	${g}_{k}=\sum_{t=0}^{H_{\tau_{k}}}
	 \nabla_{\theta}\log\pi_{\theta}(a_{t}|s_{t})R(\tau_k)|_{\theta=\theta_{k}},
	 \hat{g}_{k}=\dfrac{1}{k}\sum_{i=1}^{k}g_{i},$
	  according to Algorithm \ref{alg:Mirror-Policy-Algorithm}, at the terminal end of $k$-th episode, $k=1,2,\cdots,N$, the following holds,
	 \begin{flalign}
	 \nonumber
	 	\theta_{k+1}=\arg\min_{\theta}\{\langle -\hat{g}_k,\theta\rangle+\dfrac{1}{\alpha_k}D_{\psi}(\theta,\theta_k)\}=\mathcal{M}_{\alpha_k,\ell(-\hat{g}_{k},\theta)}^{\psi}(\theta_k).
	 \end{flalign}
	To simplify expression, let $\mathcal{J}(\theta)=-J(\theta)$.
	Then $\mathcal{J}(\theta)$ is also a $L$-smooth function, according to Eq.(\ref{app:l-mooth-inequality}), we have
	\begin{flalign}
	\nonumber
			\mathcal{J}(\theta_{k+1})&\leq\mathcal{J}(\theta_{k})+\langle\nabla\mathcal{J}(\theta)\big|_{\theta=\theta_{k}} ,\theta_{k+1}-\theta_{k}\rangle+ \dfrac{L}{2}\|\theta_{k+1}-\theta_{k}\|^{2}\\
			\nonumber
			&=\mathcal{J}(\theta_{k})-\alpha_{k}\langle\nabla\mathcal{J}(\theta_k) ,\mathcal{G}_{\alpha_{k},\ell(-\hat{g}_k,\theta_k)}^{\psi}(\theta_k)\rangle+ \dfrac{L\alpha_{k}^{2}}{2}\|\mathcal{G}_{\alpha_{k},\ell(-\hat{g}_k,\theta_k)}^{\psi}(\theta_k)\|^{2}\\
			\nonumber
			&=\mathcal{J}(\theta_{k})-\alpha_{k}\langle\hat{g}_{k} ,\mathcal{G}_{\alpha_{k},\ell(-\hat{g}_k,\theta_k)}^{\psi}(\theta_k)\rangle+ \dfrac{L\alpha_{k}^{2}}{2}\|\mathcal{G}_{\alpha_{k},\ell(-\hat{g}_k,\theta_k)}^{\psi}(\theta_k)\|^{2}
			+\alpha_{k}\langle\epsilon_{k},\mathcal{G}_{\alpha_{k},\ell(-\hat{g}_k,\theta_k)}^{\psi}(\theta_k)\rangle,
	\end{flalign}
	where $\epsilon_{k}=-\hat{g}_{k}-(-\nabla{J}(\theta_k))=-\hat{g}_{k}-\nabla{\mathcal{J}}(\theta_k)$.


	Furthermore, by Eq.(\ref{app-A-eq-3}), let $\alpha=\alpha_k$ and $g=-\hat{g}_{k}$, then we have
	\begin{flalign}
		\nonumber
			\mathcal{J}(\theta_{k+1})&\leq\mathcal{J}(\theta_{k})-\alpha_{k}\zeta\|\mathcal{G}_{\alpha_{k},\ell(-\hat{g}_k,\theta_k)}^{\psi}(\theta_k)\|^{2}+ \dfrac{L\alpha_{k}^{2}}{2}\|\mathcal{G}_{\alpha_{k},\ell(-\hat{g}_k,\theta_k)}^{\psi}(\theta_k)\|^{2}
			+\alpha_{k}\langle\epsilon_{k},\mathcal{G}_{\alpha_{k},\ell(-\hat{g}_k,\theta_k)}^{\psi}(\theta_k)\rangle\\
			\nonumber
			&=\mathcal{J}(\theta_{k})-\alpha_{k}\zeta\|\mathcal{G}_{\alpha_{k},\ell(-\hat{g}_k,\theta_k)}^{\psi}(\theta_k)\|^{2}+ \dfrac{L\alpha_{k}^{2}}{2}\|\mathcal{G}_{\alpha_{k},\ell(-\hat{g}_k,\theta_k)}^{\psi}(\theta_k)\|^{2}+\alpha_{k}\langle\epsilon_{k},\mathcal{G}_{\alpha_{k},\ell(-\nabla{J}(\theta_k),\theta_k)}^{\psi}(\theta_k)\rangle\\
			\label{app-A-eq-10}
			&~~~~~~~~~~+\alpha_{k}\langle\epsilon_{k},\mathcal{G}_{\alpha_{k},\ell(-\hat{g}_k,\theta_k)}^{\psi}(\theta_k)-\mathcal{G}_{\alpha_{k},\ell(-\nabla{J}(\theta_k),\theta_k)}^{\psi}(\theta_k)\rangle.
				\end{flalign}
Rearranging Eq.(\ref{app-A-eq-10}), we have
	\begin{flalign}
		\nonumber
		\mathcal{J}(\theta_{k+1})&\leq\mathcal{J}(\theta_{k})-\big(\zeta\alpha_k-\dfrac{L\alpha_{k}^{2}}{2}\big)\|\mathcal{G}_{\alpha_{k},\ell(-\hat{g}_k,\theta_k)}^{\psi}(\theta_k)\|^{2}+\alpha_{k}\langle\epsilon_{k},\mathcal{G}_{\alpha_{k},\ell(-\nabla{J}(\theta_k),\theta_k)}^{\psi}(\theta_k)\rangle\\
		\nonumber
		&~~~~~~~~~~+\alpha_{k}\|\epsilon_{k}\|\|\mathcal{G}_{\alpha_{k},\ell(-\hat{g}_k,\theta_k)}^{\psi}(\theta_k)-\mathcal{G}_{\alpha_{k},\ell(-\nabla{J}(\theta_k),\theta_k)}^{\psi}(\theta_k)\|.
	\end{flalign}
	In Eq.(\ref{app-A-eq-4}), let $\theta=\theta_{k},g_1=-\hat{g}_k,g_2=-\nabla{J}(\theta_k),h(x)\equiv 0$, then
	the following statement holds
	 \begin{flalign}
	 \label{app-A-eq-13}
	 	\mathcal{J}(\theta_{k+1})
	 	\leq\mathcal{J}(\theta_{k})-\big(\zeta\alpha_k-\dfrac{L\alpha_{k}^{2}}{2}\big)\|\mathcal{G}_{\alpha_{k},\ell(-\hat{g}_k,\theta_k)}^{\psi}(\theta_k)\|^{2}+\alpha_{k}\langle\epsilon_{k},\mathcal{G}_{\alpha_{k},\ell(-\nabla{J}(\theta_k),\theta_k)}^{\psi}(\theta_k)\rangle+\dfrac{\alpha_k}{L}\|\epsilon_{k}\|^{2}.
	 \end{flalign}
	 Summing the above Eq.(\ref{app-A-eq-13}) from $k=1$ to $N$ and with the condition $\alpha_k\leq\dfrac{\zeta}{L}$, we have the following statement
	  \begin{flalign}
	  \nonumber
	 \sum_{k=1}^{N}(\zeta\alpha_{k}-L\alpha^{2}_{k})
	 \|\mathcal{G}_{\alpha_{k},\ell(-\hat{g}_k,\theta_k)}^{\psi}(\theta_k)\|^{2}
	 \leq&\sum_{k=1}^{N}(\zeta\alpha_{k}-\dfrac{L\alpha^{2}_{k}}{2})
	 \|\mathcal{G}_{\alpha_{k},\ell(-\hat{g}_k,\theta_k)}^{\psi}(\theta_k)\|^{2}\\
	 \nonumber
	 \leq&\sum_{k=1}^{N}[\alpha_{k}\langle\epsilon_k,\mathcal{G}_{\alpha_{k},\ell(-\nabla{J}(\theta_k),\theta_k)}^{\psi}(\theta_k)\rangle+\dfrac{\alpha_k}{\zeta}\|\epsilon_{k}\|^{2}]+\mathcal{J}(\theta_1)-\mathcal{J}(\theta_{k+1})\\
	 \label{app-A-eq-16}
	 \leq&\sum_{k=1}^{N}[\alpha_{k}\langle\epsilon_k,\mathcal{G}_{\alpha_{k},\ell(-\nabla{J}(\theta_k),\theta_k)}^{\psi}(\theta_k)\rangle+\dfrac{\alpha_k}{\zeta}\|\epsilon_{k}\|^{2}]+\mathcal{J}(\theta_1)-\mathcal{J}(\theta^{\star}).
	 \end{flalign}
	 Recall ${g}_{k}=\sum_{t=0}^{H_{\tau_{k}}}
	 \nabla_{\theta}\log\pi_{\theta}(a_{t}|s_{t})R(\tau_k)$,  by policy gradient theorem \cite{sutton2000policy}, we have
	 $
	 	\mathbb{E}[-g_k]=-\nabla J(\theta_{k})=\nabla \mathcal{J}(\theta_{k}).
	 $
	Let $\mathcal{F}_{k}$ be the $\sigma$-field generated by all random variables defined before round $k$, $\tilde{\epsilon}_{k}=g_{k}-\nabla {J}(\theta_{k})$
	then we have: for $k=1,\cdots,N$, 
	$
	\mathbb{E}[\langle\tilde{\epsilon}_{k},\mathcal{G}_{\alpha_{k},\ell(-\nabla{J}(\theta_k),\theta_k)}^{\psi}(\theta_k)\rangle|\mathcal{F}_{k-1}]=0.
	$
	Let $\delta_{s}=\sum_{t=1}^{s}\tilde{\epsilon}_{t}$, then, for $s=1,\cdots,k$, 
	\begin{flalign}
		\label{app-A-eq-19}
		\mathbb{E}[\langle \delta_{s},\tilde{\epsilon}_{s+1} \rangle|\delta_{s}]=0.
	\end{flalign}
	Furthermore, the following statement holds
	\begin{flalign}
		\label{app-A-eq-20}
	\mathbb{E}[\|\delta_{k}\|^{2}]=\mathbb{E}[\|\delta_{k-1}\|^{2}+2\langle\delta_{k-1},\tilde{\epsilon}_{k}\rangle+\|\tilde{\epsilon}_{t}\|^{2}]\overset{(\ref{app-A-eq-19})}=\mathbb{E}[\|\delta_{k-1}\|^{2}+\|\tilde{\epsilon}_{t}\|^{2}]=\cdots=\sum_{t=1}^{k}\mathbb{E}\|\tilde{\epsilon}_{t}\|^{2}.
	\end{flalign}
	By result (\ref{def:sigma}) and Eq.(\ref{app-A-eq-20}), we have
	$
		\mathbb{E}[\|\epsilon_{k}\|^{2}]=\dfrac{1}{k^{2}}\sum_{t=1}^{k}\mathbb{E}\|\tilde{\epsilon}_{t}\|^{2}\leq \dfrac{\sigma^{2}}{k}.
	$
	Combining this result with Eq.(\ref{app-A-eq-16}), and taking expectation w.r.t $\mathcal{F}_{N}$,
	we have
	\begin{flalign}
	\nonumber
	\sum_{k=1}^{N}(\zeta\alpha_{k}-L\alpha^{2}_{k})
	\mathbb{E}[\big\|\mathcal{G}_{\alpha_{k},\ell(-\hat{g}_k,\theta_k)}^{\psi}(\theta_k)\big\|^{2}]\leq
	\mathcal{J}(\theta_1)-\mathcal{J}({\theta^{\star}})+\dfrac{\sigma^{2}}{\zeta}\sum_{k=1}^{N}\dfrac{\alpha_k}{k}.
	\end{flalign}
	 Now,
	consider the output $\tilde{\theta}_{N}=\theta_{n}$ follows the distribution of Eq.(\ref{mass-distriution}), we have
	 \begin{flalign}
    \nonumber
         \mathbb{E}[\|\mathcal{G}_{\alpha_n,\ell(-{g}_n,{\theta}_{n})}^{\psi}({\theta}_{n})\|^{2}]\leq\dfrac{\big(J(\theta^{\star})-J(\theta_1)\big){+\dfrac{\sigma^{2}}{\zeta}\sum_{k=1}^{N}{\dfrac{\alpha_k}{k}}}}{{\sum_{k=1}^{N}(\zeta\alpha_{k}-L\alpha^{2}_{k})}}
         =:
         \dfrac{\Delta{+\dfrac{\sigma^{2}}{\zeta}\sum_{k=1}^{N}{\dfrac{\alpha_k}{k}}}}{{\sum_{k=1}^{N}(\zeta\alpha_{k}-L\alpha^{2}_{k})}}
         .
    \end{flalign}
\end{proof}	
\begin{remark}	
	Particularly, if the step-size $\alpha_{k}$ is fixed to a constant:$\zeta/2L$, 
	then 	\begin{flalign}
	\nonumber
	\mathbb{E}[\|\mathcal{G}_{\alpha_n,\ell(-\hat{g}_n,\theta_n)}^{\psi}(\theta_n)\|^{2}]\leq\dfrac{4L\big(J(\theta^{\star})-J(\theta_1)\big){+2{\sigma^{2}}\sum_{k=1}^{N}{\dfrac{1}{k}}}}{N\zeta^2}.
	\end{flalign}
    Recall the following estimation
	\[\sum_{k=1}^{N}\dfrac{1}{k}=\ln N +C+o(1),\] where $C$ is the Euler constant---a positive real number and $o(1)$ is infinitesimal.
	Thus the overall convergence rate reaches $\mathcal{O}(\dfrac{\ln N}{N})$ since
	\begin{flalign}
	\nonumber
	\mathbb{E}[\|\mathcal{G}_{\alpha_n,\ell(-\hat{g}_n,\theta_n)}^{\psi}(\theta_n)\|^{2}]&\leq\dfrac{4LD^{2}_{J}+2\sigma\sum_{k=1}^{N}\dfrac{1}{k}}{N\zeta}=\mathcal{O}(\dfrac{\ln N}{N})=\widetilde{\mathcal{O}}(\dfrac{1}{N}),
	\end{flalign}
	where we use $\widetilde{\mathcal{O}}$ to hide polylogarithmic factors in the input parameters, i.e., $\widetilde{\mathcal{O}}(f(x))=\mathcal{O}(f(x)\log(f(x))^{\mathcal{O}(1)})$.
\end{remark}

\clearpage

\section{Appendix B: Short Corridor with Switched Actions (SASC)} 
\label{SCSA}

We consider the small corridor grid world which contains three sates $\mathcal{S}=\{1,2,3\}$ .
The reward is  $-1$ per step.
In each of the three nonterminal states there are only two actions, $\mathtt{right}$ and $\mathtt{left}$. These actions have their usual consequences in the state  $1$ and state  $3$ ($\mathtt{left}$ causes no movement in the first state), but in the state $2$ they are reversed. so that $\mathtt{right}$ moves to the left and $\mathtt{left}$ moves to the right. 

\begin{figure}[h]
	\centering
	\subfigure[]
	{\includegraphics[width=8cm]{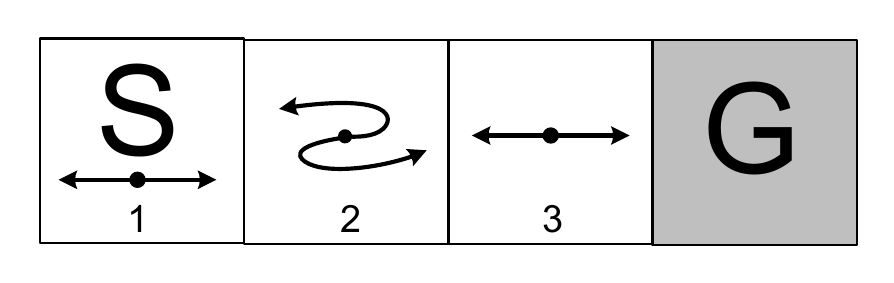}}
	\caption
	{Short corridor with switched actions (see Chapter 13 of (Sutton and Barto 2018)).
	}
\end{figure}
An action-value method with $\epsilon$-greedy action selection is forced to choose between just two policies: choosing $\mathtt{right}$ with high probability $1-\dfrac{\epsilon}{2}$ on all steps or choosing $\mathtt{left}$ with the same high probability on all time steps. 
If $\epsilon= 0.1$, then these two policies achieve a value (at the start state) of less than $-44$ and $-82$, respectively, as shown in the following graph. 
\begin{figure}[h]
	\centering
	\subfigure[]
	{\includegraphics[width=8cm]{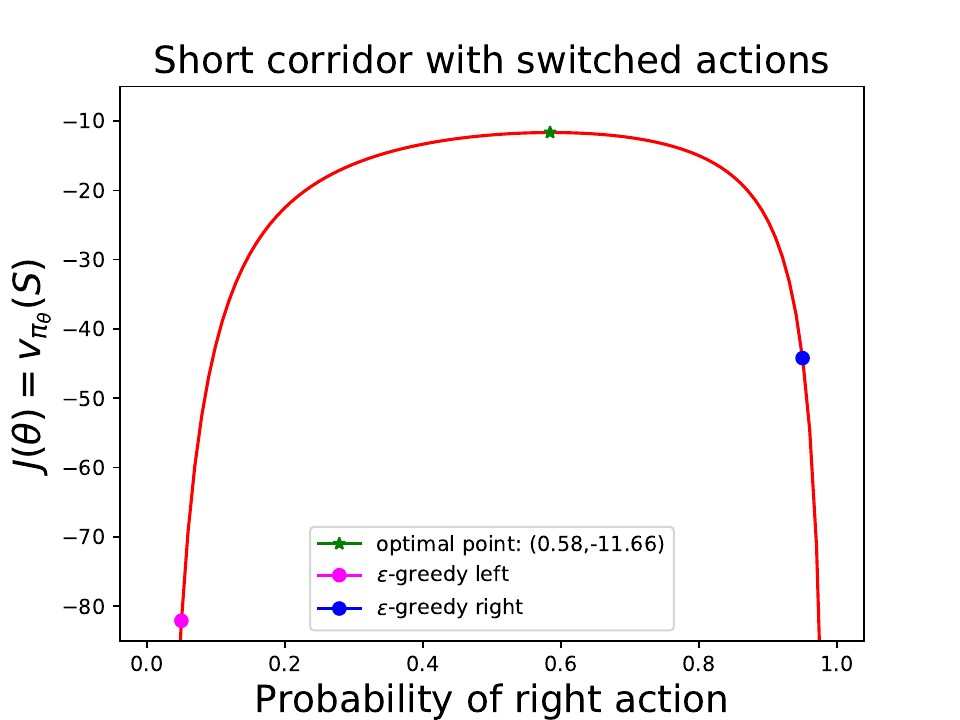}}
	\caption
	{
		The expected return $J(\theta)$ as a function of probability of right action.
		We plot this figure according to an open source code \url{https://github.com/ShangtongZhang/reinforcement-learning-an-introduction/blob/master/chapter13/short_corridor.py}.
	}
\end{figure}

A method can do significantly better if it can learn a specific probability with which to select right. The best probability is about $0.58$, which achieves a value of about $-11.6$.

\clearpage
\section{Appendix C: Proof of Theorem \ref{theo:CR-VRMP} }
\label{app: proof-of-vrmpo}
\textbf{Theorem \ref{theo:CR-VRMP}} (Convergence Rate of VRMPO)
\emph
{
	Consider $\{\tilde{\theta}_{k}\}_{k=1}^{K}$ generated by Algorithm \ref{alg:SVR-MPD}.
Under Assumption \ref{ass:On-policy-derivative}, and let $\zeta>\dfrac{5}{32}$. For any positive scalar $\epsilon$, let batch size of the trajectories of the outer loop 
\begin{flalign}
\nonumber
N_{1}&={\left(\dfrac{1}{8L\zeta^2} +\dfrac{1}{2(\zeta-\dfrac{5}{32})}\left(1+\dfrac{1}{32\zeta^2}\right)\right)}{\dfrac{\sigma^2}{\epsilon^2}},\\
\nonumber
m-1&=N_{2}={\sqrt{\left(\dfrac{1}{8L\zeta^2} +\dfrac{1}{2(\zeta-\dfrac{5}{32})}\left(1+\dfrac{1}{32\zeta^2}\right)\right)}}{\dfrac{\sigma}{\epsilon}},
\end{flalign}
the outer loop times \[K=\dfrac{8L(\mathbb{E} [\mathcal{J}(\tilde{\theta}_{0})] - \mathcal{J}(\theta^{\star}))(1+\dfrac{1}{16\zeta^2})}{{\sqrt{\big(\dfrac{1}{8L\zeta^2} +\dfrac{1}{2(\zeta-\dfrac{5}{32})}\left(1+\dfrac{1}{32\zeta^2}\right)}}\left(\zeta-\dfrac{5}{32}\right)}\dfrac{\sigma}{\epsilon},\] 
and step size $\alpha=\dfrac{1}{4L}$.
Then, Algorithm \ref{alg:SVR-MPD} outputs $\tilde{\theta}_{K}$ satisties
\begin{flalign}
\label{cr-vrmpo-bound}
\mathbb{E}\left[\|\mathcal{G}^{\psi}_{\alpha,\inner{-\nabla J(\tilde{\theta}_{K})}{\theta}}(\tilde{\theta}_{K})\|\right]\leq\epsilon.
\end{flalign}
}

\begin{lemma}
\label{app-gap}
Let $\zeta>\dfrac{5}{32}$, the batch size of the trajectories of outer loop \[N_{1}=\dfrac{\Big(\dfrac{1}{8L\zeta^2} +\dfrac{1}{2(\zeta-\dfrac{5}{32})}\big(1+\dfrac{1}{32\zeta^2}\big)\Big)\sigma^2}{\epsilon^2},\] 
the iteration times of inner loop
\[m-1=N_{2}=\dfrac{\sqrt{\Big(\dfrac{1}{8L\zeta^2} +\dfrac{1}{2(\zeta-\dfrac{5}{32})}\big(1+\dfrac{1}{32\zeta^2}\big)\Big)}\sigma}{\epsilon},\] and step size $\alpha_{k}=\dfrac{1}{4L}$.
For each $k$ and $t$, $G_{k,0}$ and $\theta_{k,0}$ are generated by Algorithm \ref{alg:SVR-MPD}, under Assumption \ref{ass:On-policy-derivative}, then the following holds,
     \begin{flalign}
     \label{app-lemma-1}
     \mathbb{E} \|\nabla \mathcal{J}(\theta_{k,0}) - G_{k,0}\|^2\leq\left(\dfrac{1}{8L\zeta^2} +\dfrac{1}{2(\zeta-\dfrac{5}{32})}\left(1+\dfrac{1}{32\zeta^2}\right)\right)^{-1}{\epsilon^2}.
     \end{flalign}
\end{lemma}
\begin{proof}
\begin{flalign}
  \mathbb{E} \|\nabla \mathcal{J}(\theta_{k,0}) - G_{k,0}\|^2&=\mathbb{E}\left\|\nabla J(\theta_{k,0})-\dfrac{1}{N_1}\sum_{i=1}^{N_1}g(\tau_{i}|\theta_{k,0})\right\|^{2}\\
  &=\dfrac{1}{N^{2}_{1}}\sum_{i=1}^{N_1}\mathbb{E}\|\nabla J(\theta_{k,0})-g(\tau_{i}|\theta_{k,0})\|^{2}\\
    \label{epsilo-1}
  &\overset{(\ref{def:sigma})}\leq\dfrac{\sigma^{2}}{N_{1}}
  =\left(\dfrac{1}{8L\zeta^2} +\dfrac{1}{2(\zeta-\dfrac{5}{32})}\left(1+\dfrac{1}{32\zeta^2}\right)\right)^{-1}{\epsilon^2}.
  \end{flalign}
\end{proof}

To simplify expression, in the following paragraph we use $\delta$ denote the long term of (\ref{epsilo-1}):
\begin{flalign}
\label{App-C:epsilo-1}
\left(\dfrac{1}{8L\zeta^2} +\dfrac{1}{2(\zeta-\dfrac{5}{32})}\big(1+\dfrac{1}{32\zeta^2}\big)\right)^{-1}{\epsilon^2}\overset{\text{def}}=\delta^2.
\end{flalign}

\subsection{(Proof of Theorem \ref{theo:CR-VRMP})}

By the definition of Bregman grdient mapping in Eq.(\ref{bregman-grdient-mapping}) and iteration (\ref{algo:VRMPG-2}), let $\alpha_{k}=\alpha$, we have
	\begin{flalign}
	\label{app-C-pre}
	\dfrac{1}{\alpha}(\theta_{k,t}-\theta_{k,t+1})\overset{(\ref{algo:VRMPG-2})}=\underbrace{\dfrac{1}{\alpha}\left(\theta_{k,t}-\arg\min_{u}\left\{\langle G_{k,t},u\rangle+\dfrac{1}{\alpha_k}D_{\psi}(u,\theta_{k,t})\right\}\right)}_{g_{k,t}}\overset{(\ref{bregman-grdient-mapping})}=\mathcal{G}^{\psi}_{\alpha,\inner{G_{k,t}}{u}}(\theta_{k,t}),
	\end{flalign}
	where we introduce $g_{k,t}$ to simplify notations. 
	
	{\color{blue}{\textbf{Step 1: Analyze the inner loop of Algorithm \ref{alg:SVR-MPD}}}}

	Now, we analyze the inner loop of Algorithm \ref{alg:SVR-MPD}. Let $\eta=\dfrac{\zeta-\dfrac{5}{32}}{4L}$, our goal is to prove
	\[
	\mathbb{E}  [\mathcal{J}(\tilde{\theta}_{k})] -  \mathbb{E}  [\mathcal{J}(\tilde{\theta}_{k-1})] \leq  -  \sum_{t=1}^{m-1}  \left(\eta \mathbb{E}[\|g_{k,t}\|^2] - \dfrac{\alpha}{2}\epsilon^2 \right),
	\]
In fact,
	\begin{flalign}
	\nonumber
	\mathcal{J}(\theta_{k,t+1}) &\overset{(\ref{app:l-mooth-inequality})}\leq \mathcal{J}(\theta_{k,t}) + \inner{\nabla \mathcal{J}(\theta_{k,t})}{\theta_{k,t+1}-\theta_{k,t}} + \dfrac{L}{2}\|\theta_{k,t+1}-\theta_{k,t}\|^2 \\
	\nonumber
	&\overset{(\ref{app-C-pre})}= \mathcal{J}(\theta_{k,t})- \alpha\inner{\nabla \mathcal{J}(\theta_{k,t})}{g_{k,t}} + \dfrac{L\alpha^2}{2}\|g_{k,t}\|^2 \\
	\nonumber
	&= \mathcal{J}(\theta_{k,t}) - \alpha\inner{\nabla \mathcal{J}(\theta_{k,t})-G_{k,t}}{g_{k,t}} -\alpha \inner{G_{k,t}}{g_{k,t}} + \dfrac{L\alpha^2}{2}\|g_{k,t}\|^2 \\
	\label{app-C-4}
	&\leq \mathcal{J}(\theta_{k,t}) + \dfrac{\alpha}{2} \|\nabla \mathcal{J}(\theta_{k,t})-G_{k,t}\|^2     -\alpha \inner{G_{k,t}}{g_{k,t}} + \left(\dfrac{L\alpha^2}{2} + \dfrac{\alpha}{2} \right) \|g_{k,t}\|^2 \\
\label{app-C-5}
	&\overset{(\ref{app-A-eq-3})}\leq \mathcal{J}(\theta_{k,t}) + \dfrac{\alpha}{2} \|\nabla \mathcal{J}(\theta_{k,t})-G_{k,t}\|^2-\zeta \alpha \|g_{k,t}\|^2 +  \left(\dfrac{L\alpha^2}{2} + \dfrac{\alpha}{2} \right) \|g_{k,t}\|^2,
	\end{flalign}
	Eq.(\ref{app-C-4}) holds due to the Cauchy-Schwarz inequality ${\displaystyle |\langle \mathbf {u} ,\mathbf {v} \rangle |\leq \|\mathbf {u} \|\|\mathbf {v} \|}
	\leq\dfrac{1}{2}(\|\mathbf {u} \|^{2}+\|\mathbf {v} \|^{2})$
	for any $\mathbf {u},\mathbf {v}\in\mathbb{R}^{n}$.
	Eq.(\ref{app-C-5}) holds if $h\equiv0$ by Eq.(\ref{app-A-eq-3}).

	Taking expectation on both sides of Eq.(\ref{app-C-5}), we have 
	\begin{flalign}
	\nonumber
	\mathbb{E}[ \mathcal{J}(\theta_{k,t+1})] &\leq \mathbb{E}[\mathcal{J}(\theta_{k,t})] + \dfrac{\alpha}{2}\mathbb{E}\Big[ \|\nabla \mathcal{J}(\theta_{k,t})-G_{k,t}\|^2\Big]- \left(\zeta \alpha-\dfrac{L\alpha^2}{2} - \dfrac{\alpha}{2} \right) \mathbb{E}\Big[\|g_{k,t}\|^2\Big]\\
		\nonumber
	&\leq \mathbb{E}[\mathcal{J}(\theta_{k,t})]  +\dfrac{\alpha }{2}\sum_{i=1}^{t} \dfrac{L^2}{N_2}\mathbb{E}\|\theta_{k,i+1} - \theta_{k,i}\|^2 \\
	\label{app-C-6}
	&~~~~~~~~+  \dfrac{\alpha}{2}\mathbb{E}\|G_{k-1,0} - \nabla \mathcal{J}({\theta}_{k-1,0})\|^2   - \left(\zeta \alpha - \dfrac{L\alpha^2}{2} - \dfrac{\alpha}{2} \right) \mathbb{E} \|g_{k,t}\|^2
	\end{flalign}
	Eq.(\ref{app-C-6}) holds due to since: 
	\begin{flalign}
	\label{app-c-lemma-2}
			\mathbb{E}[\|G_{k,t}-\nabla\mathcal{J}(\theta_{k,t})\|^{2}]\leq\sum_{i=1}^{t}\dfrac{L^{2}}{N_2}
				\mathbb{E}[\|\theta_{k,i+1}-\theta_{k,i}\|^{2}]+\mathbb{E}[\|G_{k-1,0}-\nabla\mathcal{J}(\tilde{\theta}_{k-1})\|^{2}].
	\end{flalign}

	By Lemma \ref{app-gap}, Eq.(\ref{app-C-6}) and Eq.(\ref{app-C-pre}), we have
	\begin{flalign}
	\nonumber
		\mathbb{E}[ \mathcal{J}(\theta_{k,t+1})] &\leq\mathbb{E}[\mathcal{J}(\theta_{k,t})]+\dfrac{\alpha^{3}L^{2}}{2N_2}\sum_{i=1}^{t}\mathbb{E}\Big[\|g_{k,i}\|^{2}\Big]+\dfrac{\alpha \delta^{2}}{2}- \left(\zeta \alpha - \dfrac{L\alpha^2}{2} - \dfrac{\alpha}{2} \right) \mathbb{E} \|g_{k,t}\|^2,
	\end{flalign}
where $\epsilon_1$ is defined in (\ref{App-C:epsilo-1}).
Recall the parameter $\tilde{\theta}_{k-1}={\theta}_{k-1,m}$ is generated by the last time of  $(k-1)$-th episode, now, we consider the following equation 
	\begin{flalign}
	\nonumber
	&\mathbb{E}  [\mathcal{J}(\theta_{k,t+1})] -  \mathbb{E}  [\mathcal{J}(\tilde{\theta}_{k-1})] 
	 \\
	 \nonumber
	\leq&  \dfrac{\alpha^3 L^{2}}{2N_2} \sum_{j=1}^{t} \sum_{i=1}^{j}\mathbb{E}\|g_{k,i}\|^2 +\dfrac{\alpha}{2}\sum_{j=1}^{t} \delta^2 - \left(\zeta \alpha - \dfrac{L\alpha^2}{2} - \dfrac{\alpha}{2} \right) \sum_{j=1}^{t}\mathbb{E}\|g_{k,j}\|^2 
	 \\ 
	\nonumber
	\leq&  \dfrac{\alpha^3 L^{2}}{2N_2} \sum_{j=1}^{t} \sum_{i=1}^{t}\mathbb{E}\|g_{k,i}\|^2 +\dfrac{\alpha}{2}\sum_{j=1}^{t} \delta^2 - \left(\zeta \alpha - \dfrac{L\alpha^2}{2} - \dfrac{\alpha}{2} \right) \sum_{j=1}^{t}\mathbb{E}\|g_{k,j}\|^2 
	\\
	\nonumber
	=&  \dfrac{\alpha^3 L^{2}t}{2N_2}  \sum_{i=1}^{t}\mathbb{E}\|g_{k,i}\|^2 +\dfrac{\alpha}{2}\sum_{j=1}^{t} \delta^2 - \left(\zeta \alpha - \dfrac{L\alpha^2}{2} - \dfrac{\alpha}{2} \right) \sum_{j=1}^{t}\mathbb{E}\|g_{k,j}\|^2 \\
	\label{app-C-12}
\leq&  \dfrac{\alpha^3 L^{2}(m-1)}{2N_2}  \sum_{i=1}^{t}\mathbb{E}\|g_{k,i}\|^2 +\dfrac{\alpha}{2}\sum_{j=1}^{t} \delta^2 - \left(\zeta \alpha - \dfrac{L\alpha^2}{2} - \dfrac{\alpha}{2} \right) \sum_{j=1}^{t}\mathbb{E}\|g_{k,j}\|^2 
	\\
	\nonumber
	= &\dfrac{\alpha}{2}\sum_{j=1}^{t} \delta^2 - \underbrace{\left(\zeta \alpha - \dfrac{L\alpha^2}{2} - \dfrac{\alpha}{2} -\dfrac{\alpha^3 L^{2}(m-1)}{2N_2}\right)}_{\overset{\text{def}}=\eta=\dfrac{\zeta-\dfrac{5}{32}}{4L}} \sum_{j=1}^{t}\mathbb{E}\|g_{k,j}\|^2 \\
	\label{app-C-13}
	=&  -  \sum_{i=1}^{t}  \left(\eta \mathbb{E}[\|g_{k,i}\|^2] - \dfrac{\alpha}{2}\underbrace{\left(\left(\dfrac{1}{8L\zeta^2} +\dfrac{1}{2\left(\zeta-\dfrac{5}{32}\right)}\left(1+\dfrac{1}{32\zeta^2}\right)\right)\right)^{-1}{\epsilon^2}}_{=\delta^2} \right),
	\end{flalign}
	Eq.(\ref{app-C-12}) holds due to $t\leq m-1$. 
	
	If $t=m-1$, then the last Eq.(\ref{app-C-13}) implies
\begin{flalign}
	\label{app-C-17}
	\mathbb{E}  [\mathcal{J}(\tilde{\theta}_{k})] -  \mathbb{E}  [\mathcal{J}(\tilde{\theta}_{k-1})] \leq  -  \sum_{t=1}^{m-1}  \left(\eta \mathbb{E}[\|g_{k,t}\|^2] - \dfrac{\alpha}{2}\delta^2 \right).
\end{flalign}

{\color{blue}{\textbf{Step 2: Analyze the outer loop of Algorithm \ref{alg:SVR-MPD}}}}

We now consider the output of Algorithm \ref{alg:SVR-MPD},
\begin{flalign}
\nonumber
\mathbb{E} [\mathcal{J}(\tilde{\theta}_{K})] - \mathbb{E}[\mathcal{J}(\tilde{\theta}_{0})] 
\nonumber
=& \Big(\mathbb{E} [\mathcal{J}(\tilde{\theta}_{1})] - \mathbb{E}[\mathcal{J}(\tilde{\theta}_{0})]\Big) + \Big(\mathbb{E} [\mathcal{J}(\tilde{\theta}_{2})] - \mathbb{E}[\mathcal{J}(\tilde{\theta}_{1})]\Big)\\
\nonumber
 &~~~~~~~~~~~~~~~~~~~~~+\cdots + \Big(\mathbb{E} [\mathcal{J}(\tilde{\theta}_{K})] - \mathbb{E}[\mathcal{J}(\tilde{\theta}_{K-1})]\Big)\\
\nonumber
\overset{(\ref{app-C-17})}\leq& -   \sum_{t=0}^{m-1} \left(\eta \mathbb{E}\|g_{1,t}\|^2 - \dfrac{\alpha  }{2}\delta^2 \right) - \sum_{t=0}^{m-1}\left(\eta \mathbb{E}\|g_{2,t}\|^2 - \dfrac{\alpha}{2}\delta^2 \right)\\
\nonumber
&~~~~~~~~~~~~~~~~~~~~~~~~
- \cdots -   \sum_{t=0}^{m-1} \left(\eta \mathbb{E}\|g_{K,t}\|^2 - \dfrac{\alpha  }{2}\delta^2 \right)\\
\nonumber
= &-   \sum_{k=1}^{K}\sum_{t=0}^{m-1}  \left(\eta \mathbb{E}\|g_{k,t}\|^2 - \dfrac{\alpha}{2}\delta^2 \right) \\
\nonumber
= &-   \sum_{k=1}^{K} \sum_{t=1}^{m-1}\Big(\eta \mathbb{E}\|g_{k,t}\|^2\Big) + \dfrac{ K \alpha}{2}\delta^2 ,
\end{flalign}
then we have 
\begin{flalign}
	\label{app-C-18}
	\sum_{k=1}^{K} \sum_{t=1}^{m-1}\Big(\eta \mathbb{E}\|g_{k,t}\|^2\Big) \leq \mathbb{E} [\mathcal{J}(\tilde{\theta}_{0})] - \mathcal{J}(\theta^{\star}) + \dfrac{ K(m-1) \alpha}{2}\delta^2.
\end{flalign}
Recall the notation in Eq.(\ref{app-C-pre})
\[{g_{k,t}}=\dfrac{1}{\alpha}(\theta_{k,t}-\arg\min_{u}\{\langle G_{k,t},u\rangle+\dfrac{1}{\alpha}D_{\psi}(u,\theta_{k,t})\})=\mathcal{G}^{\psi}_{\alpha,\inner{G_{k,t}}{u}}(\theta_{k,t}),\] 
and we introduce following $\tilde{g}(\theta_{k,t})$ to simplify notations,
\begin{flalign}
\nonumber
	\tilde{g}(\theta_{k,t})&=\mathcal{G}^{\psi}_{\alpha,\inner{-\nabla J(\theta_{k,t})}{u}}(\theta_{k,t})\overset{\text{def}}=\tilde{g}_{k,t}\\
	\label{app-C-20}
	&=\dfrac{1}{\alpha}\left(\theta_{k,t}-\arg\min_{u}\left\{\langle {-\nabla J(\theta_{k,t})},u\rangle+\dfrac{1}{\alpha}D_{\psi}(u,\theta_{k,t})\right\}\right).
\end{flalign}
Then, the following holds
\begin{flalign}
	\nonumber
	\mathbb{E} \|\tilde{g}_{k,t}\|^2 \leq& \mathbb{E} \| g_{k,t} \|^2 + \mathbb{E} \|\tilde{g}_{k,t}- g_{k,t}\|^2 \\
	\label{app-C-22}
	\overset{(\ref{app-A-eq-4})}\leq& \mathbb{E} \| g_{k,t}  \|^2 + \dfrac{1}{\zeta^2} \mathbb{E} \|\nabla \mathcal{J}(\theta_{k,t}) - G_{k,t}\|^2 , 
\end{flalign}
Eq.(\ref{app-C-22}) holds due to the Eq.(\ref{app-A-eq-4}).

Let $\nu$ be the number that is selected randomly from $\{1,\cdots,(m-1)K\}$ which is the output of Algorihtm \ref{alg:SVR-MPD}, for the convenience of proof the there is no harm in hypothesis that $\nu=k\cdot (m-1)+t$ and we denote the output $\theta_{\nu}=\theta_{k,t}$.

Now, we analyze above Eq.(\ref{app-C-22}) and show it is bounded as following two parts (\ref{app-C-24}) and (\ref{app-C-28})
\begin{flalign}
	\label{app-C-23}
	\mathbb{E} \|g(\theta_{\nu})\|^2
= \dfrac{1}{(m-1)K}\sum_{k=1}^{K} \sum_{t=1}^{m-1}\mathbb{E} \|g_{k,t}\|^2 \overset{(\ref{app-C-18})}\leq \dfrac{\mathbb{E} [\mathcal{J}(\tilde{\theta}_{0})] - \mathcal{J}(\theta^{\star}) }{(m-1)K\eta }  +  \dfrac{\alpha }{2\eta}\delta^2,
\end{flalign}
which implies the following holds
\begin{flalign}
\label{app-C-24}
\mathbb{E} \|g_{k,t}\|^2
\leq \dfrac{\mathbb{E} [\mathcal{J}(\tilde{\theta}_{0})] - \mathcal{J}(\theta^{\star}) }{(m-1)K\eta }  +  \dfrac{\alpha }{2\eta}\delta^2.
\end{flalign}
For another part of Eq.(\ref{app-C-22}), notice $\nu=k(m-1)+t$, then we have 
\begin{flalign}
\label{app-C-25}
\mathbb{E}  \|\nabla \mathcal{J}(\theta_{k,t})-G_{k,t}\|^2  
=&\mathbb{E}  \|\nabla \mathcal{J}(\theta_{\nu})-G_{\nu}\|^2 
 \\
 \nonumber
 \overset{(\ref{app-c-lemma-2})}\leq&  \mathbb{E}\Bigg[\dfrac{L^2}{N_2}\sum_{i=1}^{t } \mathbb{E}\|\theta_{k,i+1} - \theta_{k,i}\|^2 + \mathbb{E}[\|G_{k-1,0}-\nabla\mathcal{J}(\tilde{\theta}_{k-1})\|^{2}] \Bigg]\\
 \nonumber
\overset{(\ref{app-lemma-1})}\leq&  \mathbb{E}\Bigg[\dfrac{L^2}{N_2}\sum_{i=1}^{t } \mathbb{E}\|\theta_{k,i+1} - \theta_{k,i}\|^2 + \dfrac{\alpha}{2}\delta^2 \Bigg]
\\
\nonumber
\overset{(\ref{app-C-pre})}=& \mathbb{E}\Bigg[\dfrac{L^2\alpha^2}{ N_2}\sum_{i=1}^{t } \mathbb{E}\|g_{k,i}\|^2\Bigg] +\dfrac{\alpha}{2} \delta^2 
\\
\nonumber
\overset{t\leq m}\leq& \mathbb{E}\Bigg[\dfrac{L^2\alpha^2}{ N_2}\sum_{i=1}^{m-1} \mathbb{E}\|g_{k,i}\|^2\Bigg] + \dfrac{\alpha}{2}\delta^2 
\\
\label{app-C-26}
\leq&  \dfrac{L^2\alpha^2}{K N_2}\sum_{k=1}^{K} \sum_{t=1}^{m-1}\mathbb{E} \|g_{k,t}\|^2 + \dfrac{\alpha}{2}\delta^2 \\
\label{app-C-28}
\overset{(\ref{app-C-18})} \leq &\dfrac{ L^2 \alpha^2}{KN_2\eta} \left( \mathbb{E} [\mathcal{J}(\tilde{\theta}_{0})] - \mathcal{J}(\theta^{\star})  \right) +  \Big(\dfrac{ L^2 \alpha^3(m-1)}{2 N_2\eta}+\dfrac{\alpha}{2}\Big)\delta^{2}, 
\end{flalign}
Eq.(\ref{app-C-26}) holds due to the fact that the probability of selecting $\nu=k\cdot (m-1) +t$ is less than $\dfrac{1}{K}$.

Taking Eq(\ref{app-C-23}) and Eq.(\ref{app-C-26}) into Eq.(\ref{app-C-22}), then we have the following inequity
\begin{flalign}
\nonumber
\mathbb{E} \|\tilde{g}_{k,t}\|^2 &\leq\Big(\dfrac{1}{(m-1)K\eta}+\dfrac{ L^2 \alpha^2}{KN_2\eta\zeta^{2}}\Big) \left( \mathbb{E} [\mathcal{J}(\tilde{\theta}_{0})] - \mathcal{J}(\theta^{\star})  \right) +  \Big(\dfrac{ L^2 \alpha^3(m-1)}{2 N_2\eta\zeta^{2}}+\dfrac{\alpha}{2\zeta^{2}}+\dfrac{\alpha}{2\eta}\Big)\delta^{2}.
\end{flalign}
Recall $\alpha=\dfrac{1}{4L}$, \[N_{1}=\dfrac{\left(\dfrac{1}{8L\zeta^2} +\dfrac{1}{2(\zeta-\dfrac{5}{32})}\left(1+\dfrac{1}{32\zeta^2}\right)\right)\sigma^2}{\epsilon^2},\] 
\[N_{2}=m-1=\dfrac{\sqrt{\left(\dfrac{1}{8L\zeta^2} +\dfrac{1}{2(\eta-\dfrac{5}{32})}\left(1+\dfrac{1}{32\zeta^2}\right)\right)}\sigma}{\epsilon},\] then we have
\begin{flalign}
\mathbb{E}\|\mathcal{G}_{\alpha,\langle-\nabla J (\tilde{\theta}_{K}),{\theta}\rangle}\|^2=\mathbb{E} \|\tilde{g}_{k,t}\|^2 \leq\dfrac{4L}{K(m-1)(\zeta-\dfrac{5}{32})}\big(1+\dfrac{1}{16\zeta^2}\big)(\mathbb{E} [\mathcal{J}(\tilde{\theta}_{0})] - \mathcal{J}(\theta^{\star}))+\dfrac{1}{2}\epsilon^2.
\end{flalign}
Furthermore, since \[K=\dfrac{8L\big(1+\dfrac{1}{16\zeta^2}\big)}{(m-1)(\zeta-\dfrac{5}{32})}\cdot\dfrac{\mathbb{E} [\mathcal{J}(\tilde{\theta}_{0})] - \mathcal{J}(\theta^{\star})}{\epsilon^2}=\dfrac{8L(\mathbb{E} [\mathcal{J}(\tilde{\theta}_{0})] - \mathcal{J}(\theta^{\star}))(1+\dfrac{1}{16\zeta^2})}{{\sqrt{\big(\dfrac{1}{8L\zeta^2} +\dfrac{1}{2(\zeta-\dfrac{5}{32})}\big(1+\dfrac{1}{32\zeta^2}\big)\big)}}\big(\zeta-\dfrac{5}{32}\big)}\dfrac{\sigma}{\epsilon},\] we have
\begin{flalign}
\mathbb{E}[\|\mathcal{G}^{\psi}_{\alpha,\inner{-\nabla J(\tilde{\theta}_{K})}{\theta}}(\tilde{\theta}_{K})\|]\leq\epsilon.
\end{flalign}

\section{On-line VRMPO}

\begin{algorithm}[t]
    \caption{On-line VRMPO}
    \label{on-line VRMPO}
    \begin{algorithmic}[1]
        \STATE {\bfseries  Initialize:} Policy $\pi_{\theta}(a|s)$ with parameter $\tilde{\theta}_{0}$, critic network $Q_{\omega^{j}}$ with parameter $\tilde{\omega}^{j}_{0}$, $j=1,2$, mirror map $\psi$,step-size $\alpha> 0$, epoch size $K$,$m$.
        \STATE {\bfseries  Initialize:} Parameter $\tilde{\omega}^{j}_{0},j=1,2$ ,$0<\kappa<1$ .      
        \FOR{$k=1$ {\bfseries to} $K$}
        \FOR{each domain step}
        \STATE $a_{t}\sim\pi_{\tilde{\theta}_{k-1}}(\cdot|s_{t})$
        \STATE $s_{t+1}\sim P(\cdot|s_{t},a_{t})$
        \STATE $\mathcal{D}=\mathcal{D}\cup\{(s_t,a_t,r_{t},s_{t+1})\}$
        \ENDFOR
        \STATE sample mini-batch $\{(s_i,a_i)\}_{i=1}^{N}\sim\mathcal{D}$
        \STATE $\theta_{k,0}=\tilde{\theta}_{k-1}, \omega_{k,0}=\tilde{\omega}^{j}_{k-1},j=1,2$
        \STATE $L_{\theta}(s,a)=-\log\pi_{\theta}(s,a)\underbrace{(\min_{j=1,2}Q_{\omega^{j}_{k-1}}(s,a))}_{\text{Double Q-Learning~\cite{van2016deep}}}$
        \STATE $\theta_{k,1}=\theta_{k,0}-\alpha_k G_{k,0}$,~~~where $G_{k,0}=\dfrac{1}{N}\sum_{i=1}^{N}\nabla_{\theta}L_{\theta}(s_i,a_i)\Big|_{\theta=\theta_{k,0}}$
        \FOR{$t=1$ {\bfseries to} $m-1$}
        \STATE  /* Update Actor $(m-1)$ Epochs  */
         \STATE ~~~~~~~~~~~~~~~~~~~sample mini-batch $\{(s_i,a_i)\}_{i=1}^{N}\sim\mathcal{D}$       
        \begin{flalign}
        &\delta_{k,t}=\dfrac{1}{N}\sum_{i=1}^{N}\nabla_{\theta}L_{\theta}(s_i,a_i)\Big|_{\theta=\theta_{k,t}}-\dfrac{1}{N}\sum_{i=1}^{N}\nabla_{\theta}L_{\theta}(s_i,a_i)\Big|_{\theta=\theta_{k,t-1}}\\
        &G_{k,t}=\delta_{k,t}+G_{k,t-1}\\
        \label{app-C-theta-iteration}
        &\theta_{k,t+1}=\arg\min_{u}\{\langle G_{k,t},u\rangle+\dfrac{1}{\alpha_k}D_{\psi}(u,\theta_{k,t})\}
        \end{flalign}
        \ENDFOR
        \FOR{$t=1$ {\bfseries to} $m-1$}
        \STATE /* Update Critic $(m-1)$ Epochs */
        \STATE ~~~~~~~~~~~~~~~~~~~sample mini-batch $\{(s_i,a_i)\}_{i=1}^{N}\sim\mathcal{D}$ 
        \begin{flalign}
        \label{Loss-of-critic}
        &L_{\omega^{j}_{k-1,t-1}}(\omega)=\dfrac{1}{N}\sum_{i=1}^{N}(Q_{\omega^{j}_{k-1,t-1}}(s_i,a_i)-Q_{\omega}(s_i,a_i))^{2},j=1,2\\
        \label{learning-para-of-critic}
        &\omega^{j}_{k,t}=\arg\min_{\omega}L_{\omega^{j}_{k-1,t-1}}(\omega),j=1,2
        \end{flalign}
        \ENDFOR
        \STATE $\tilde{\theta}_{k}\overset{\text{def}}=\theta_{k,m-1}$
        \STATE $\tilde{\omega}^{j}_{k}\overset{\text{def}}=\omega^{j}_{k,m-1},j=1,2$
        \STATE /* Soft Update */
        \STATE $\tilde{\theta}_{k}\leftarrow\kappa\tilde{\theta}_{k-1}+(1-\kappa)\tilde{\theta}_{k}$
        \STATE $\tilde{\omega}^{j}_{k}\leftarrow\kappa\tilde{\omega}^{j}_{k-1}+(1-\kappa)\tilde{\omega}^{j}_{k},j=1,2$
        \ENDFOR
       \STATE {\bfseries Output:} $\tilde{\theta}_{K},\{\tilde{\omega}^{j}_{K}\}_{j=1,2}$.
    \end{algorithmic}
\end{algorithm}

\clearpage
\section{Experiments}
\label{app-c-para}

\subsection{Additional experiments}
We compare the VRMPO with baseline algorithms on test scores and max-return. All the results are shown in Figure \ref{figure-6}-\ref{figure-7}.
\begin{figure*}[h]
      \label{learning-curves}
    \centering
    \subfigure[Walker2d-v2]
    {\includegraphics[width=5cm,height=4cm]{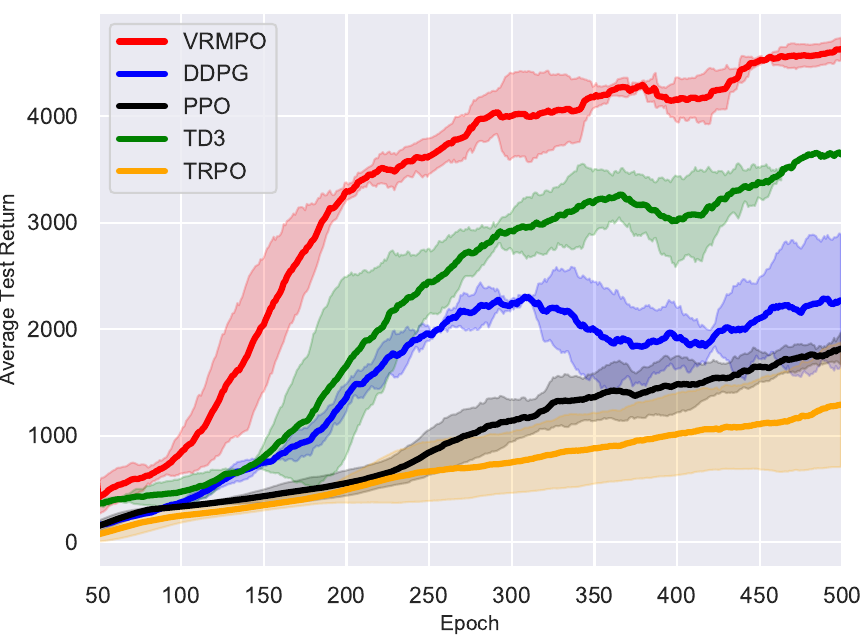}}
    \subfigure[HalfCheetah-v2]
    {\includegraphics[width=5cm,height=4cm]{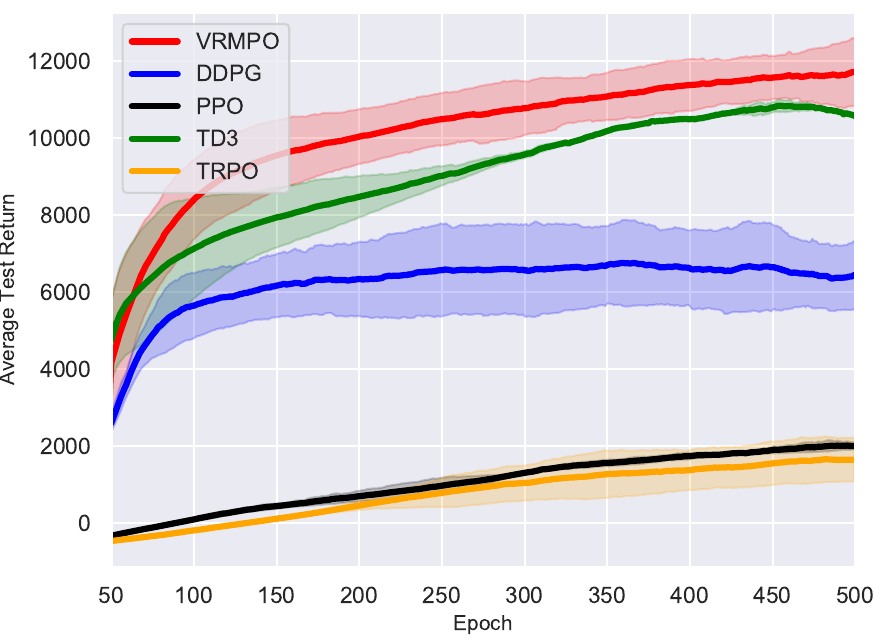}}
    \subfigure[Reacher-v2]
    {\includegraphics[width=5cm,height=4cm]{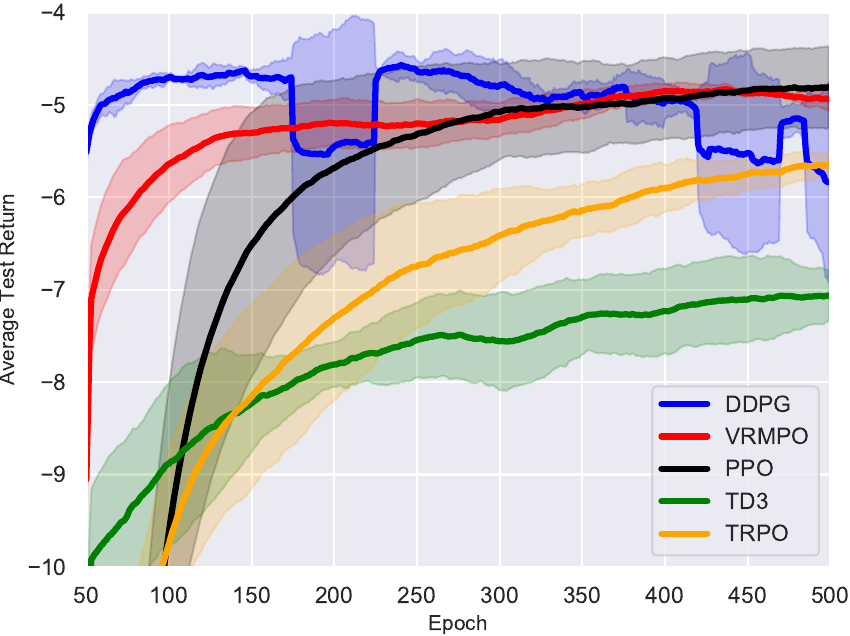}}
    \subfigure[Hopper-v2]
    {\includegraphics[width=5cm,height=4cm]{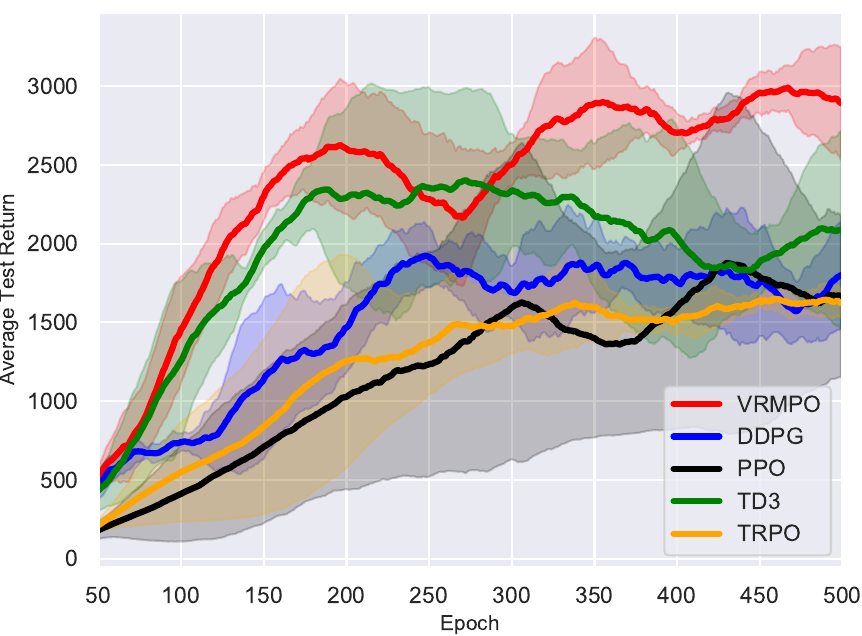}}
    \subfigure[InvDoublePendulum-v2]
    {\includegraphics[width=5cm,height=4cm]{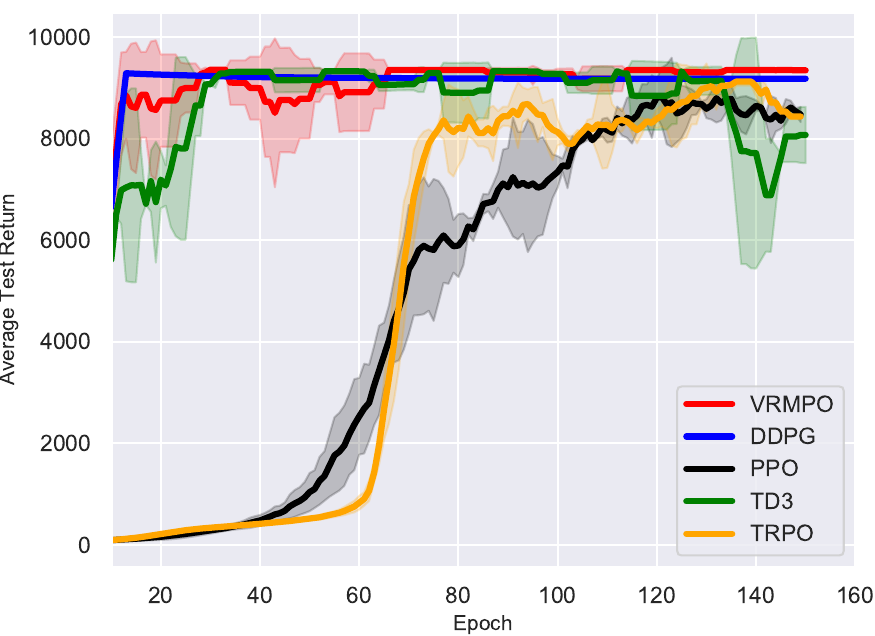}}
    \subfigure[InvPendulum-v2]
    {\includegraphics[width=5cm,height=4cm]{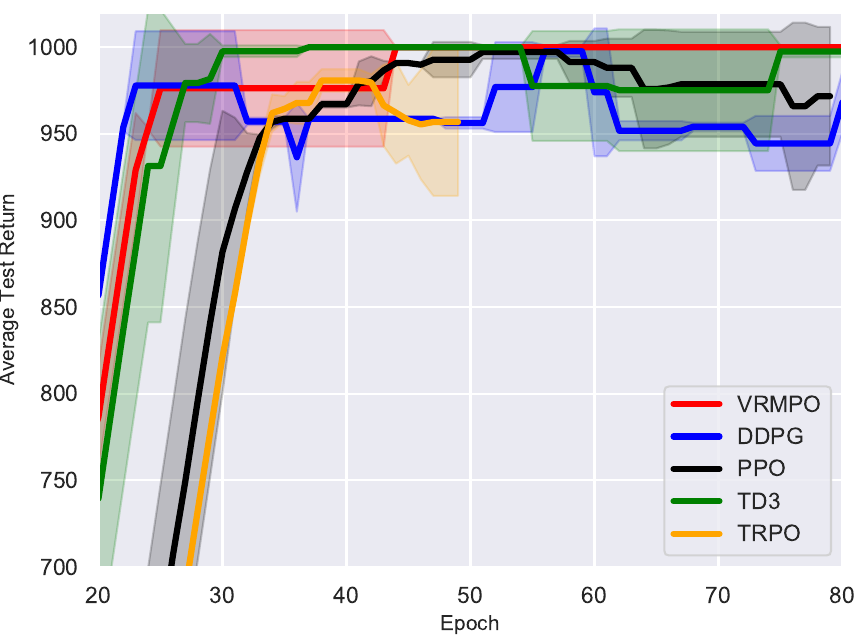}}
    \caption
    {
        Learning curves of test score over the epoch, where we run 5000 iterations for each epoch. The shaded region represents the standard deviation of the test score over the best 3 trials. Curves are smoothed uniformly for visual clarity.
    }
    \label{figure-6}
\end{figure*}
\begin{figure*}[h]
      \label{learning-curves}
    \centering
    \subfigure[Walker2d-v2]
    {\includegraphics[width=5cm,height=4cm]{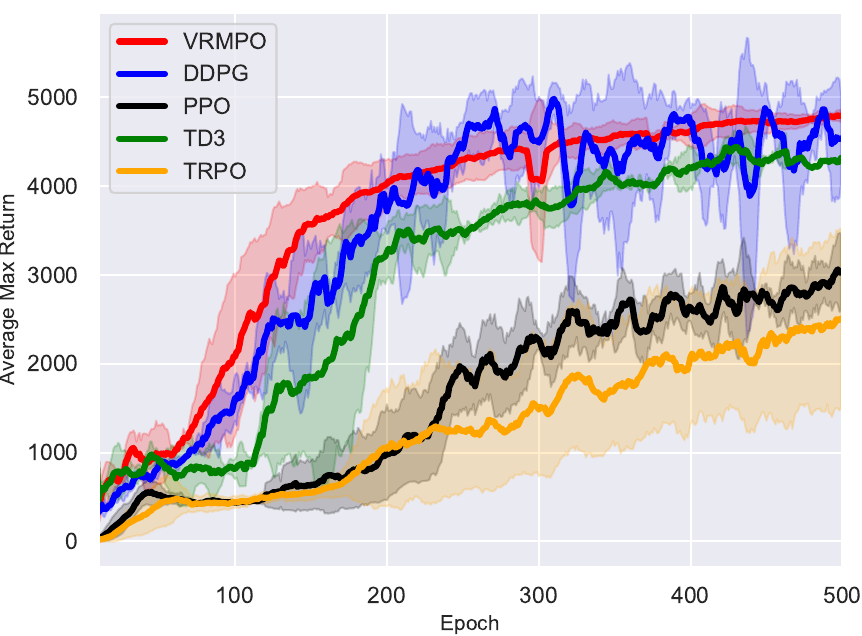}}
    \subfigure[HalfCheetah-v2]
    {\includegraphics[width=5cm,height=4cm]{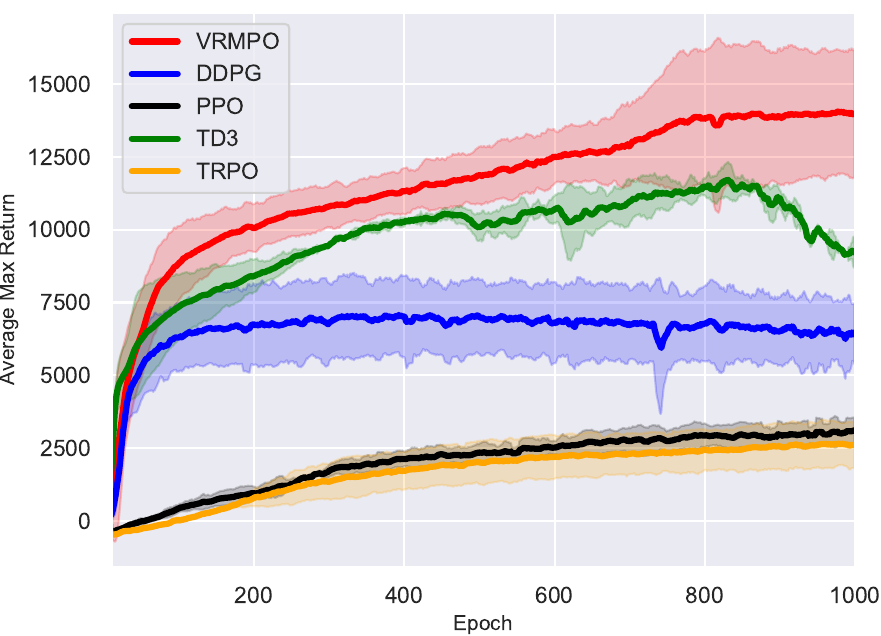}}
    \subfigure[Reacher-v2]
    {\includegraphics[width=5cm,height=4cm]{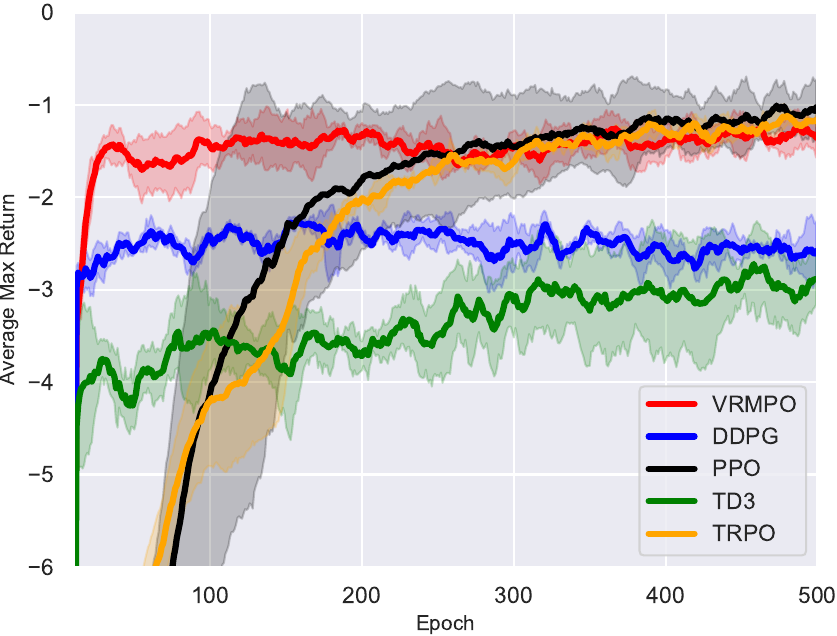}}
    \subfigure[Hopper-v2]
    {\includegraphics[width=5cm,height=4cm]{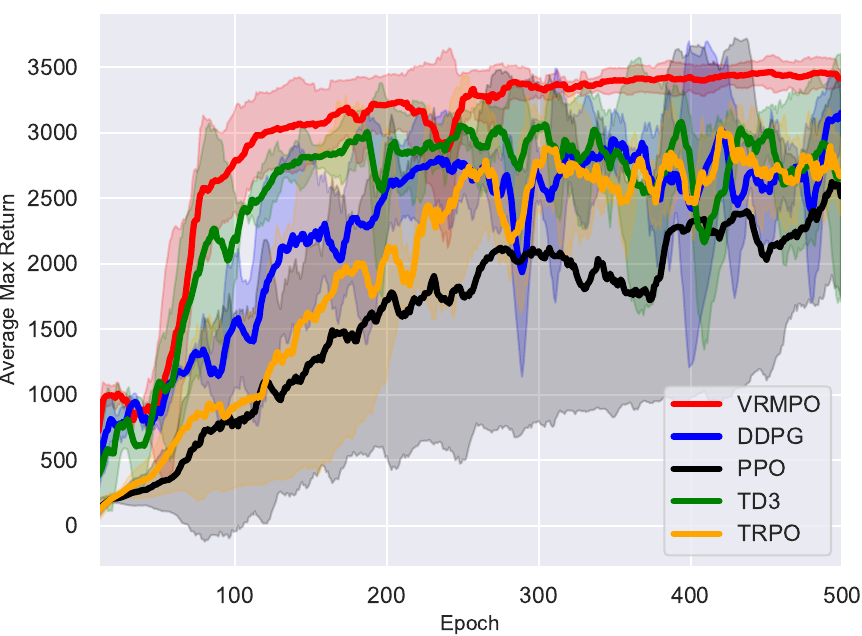}}
    \subfigure[InvDoublePendulum-v2]
    {\includegraphics[width=5cm,height=4cm]{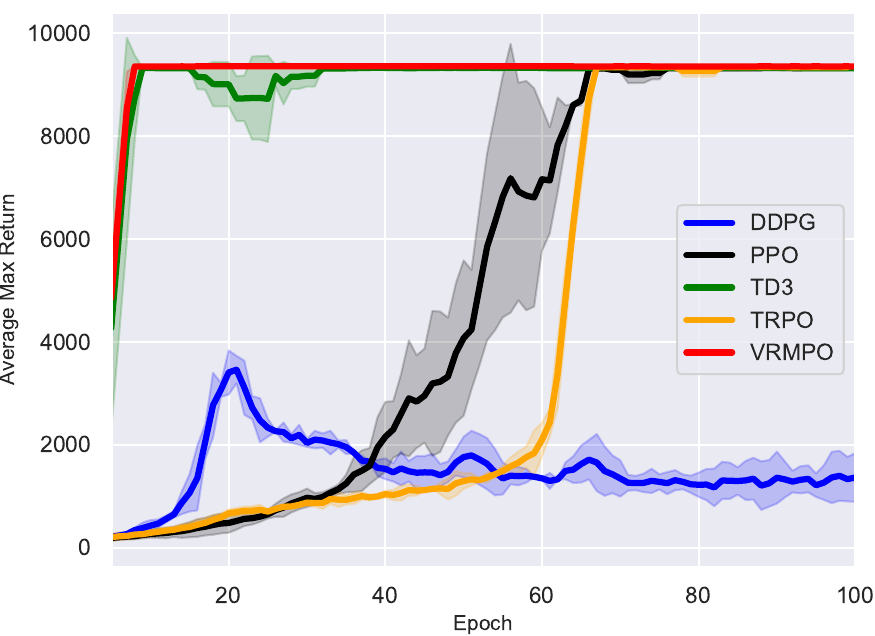}}
    \subfigure[InvPendulum-v2]
    {\includegraphics[width=5cm,height=4cm]{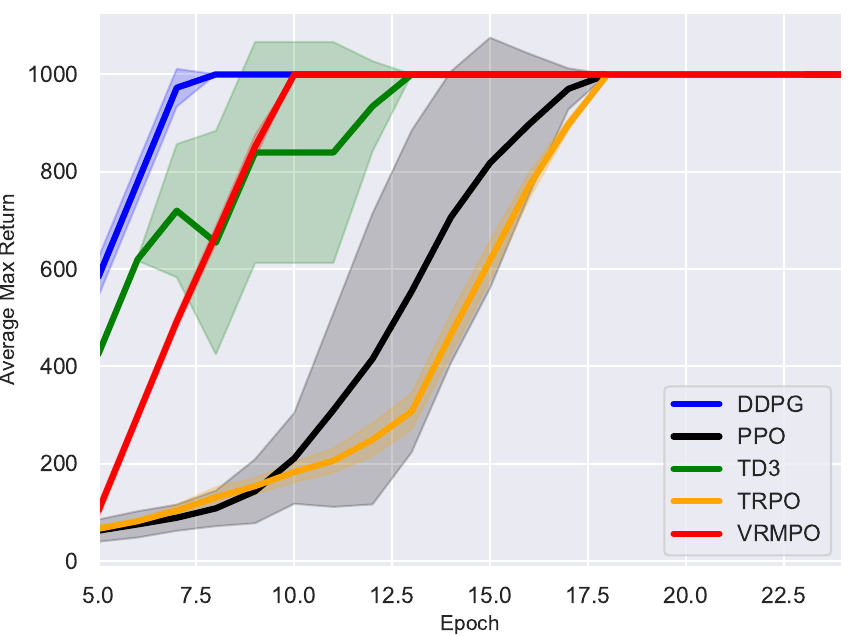}}
    \caption
    {
        Learning curves of max-return over the epoch, where we run 5000 iterations for each epoch. The shaded region represents the standard deviation of the test score over the best 3 trials. Curves are smoothed uniformly for visual clarity.
    }
      \label{figure-7}
\end{figure*}

\subsection{E.1: Some Practical Tricks for an On-line Implementation of VRMPO}
\label{app-ex:on-line-vrmpo}
In this section, we present the details of the practical tricks for an on-line implementation of $\mathtt{VRMPO}$ that extends the previous Algorithm \ref{alg:SVR-MPD}. 
We have provided a complete implementation in Algorithm \ref{on-line VRMPO}.

It is noteworthy that the policy gradient (\ref{Eq:G_t}) of $\mathtt{VRMPO}$ is an off-line (i.e., known as the Monte Carlo method) estimator as the standard $\mathtt{REINFOECE}$.
As pointed by \cite{sutton2018reinforcement}, $\mathtt{REINFOECE}$ converges asymptotically to a local minimum, but like all Monte Carlo methods it tends to learn slowly, 
to be inconvenient for continuous control tasks, and it is limited in the application to some complex domains. 
This could also happen in $\mathtt{VRMPO}$.
To eliminate above inconveniences and to gain the advantages of $\mathtt{VRMPO}$ to complex tasks, we introduce some practical tricks for on-line implementation of $\mathtt{VRMPO}$.

\textbf{(i)} Firstly, we extend Algorithm \ref{alg:SVR-MPD} to be an actor-critic \cite{konda2000actor} structure, i.e., we introduce a critic structure to Algorithm \ref{alg:SVR-MPD}.
Concretely, for each step $t$, we construct a critic network (an estimator of action-value) $Q_{\omega}(s,a)$ with the parameter $\omega$,
sample $\{(s_i,a_i)\}_{i=1}^{N}$ from a data memory $\mathcal{D}$,
and learn the parameter $\omega$ via minimizing the following fundamental critic loss:
\begin{flalign}
\label{Loss-of-critic-1}
L_{\omega}=\dfrac{1}{N}\sum_{i=1}^{N}(Q_{\omega_{k-1}}(s_i,a_i)-Q_{\omega}(s_i,a_i))^{2}.
\end{flalign}
For the complex real-world domains, we should tune necessitate meticulous hyper-parameter.
In order to improve sample efficiency, we draw on the technique of $\mathtt{Double~Q\text{-}learning}$~\cite{van2016deep} to $\mathtt{VRMPO}$.
For more details, please see $\mathtt{Line~17}$-$\mathtt{20}$ of Algorithm \ref{on-line VRMPO}.

For the implementation of critic network $Q_{\omega}(s,a)$, we use a two-layer feedforward neural network of 400 and 300 hidden nodes respectively, with rectified linear units (ReLU) between each layer, and a final tanh unit following the output of the critic network $Q_{\omega}(s, a)$. 
After coding the critic network $Q_{\omega}(s,a)$, we get the loss of critic $L_{\omega}$ (\ref{Loss-of-critic-1})/(\ref{Loss-of-critic}),
we use adaptive strategies to compute the step-size according to ADAptive Moment estimation (ADAM) \cite{kingma2015adam} to learning the parameter $\omega$.
Concretely, for (\ref{learning-para-of-critic}), we use Tensorflow to return the parameter $\omega$:
$\mathtt{tf.train.AdamOptimizer(learning\_rate=lr).minimize}(L_{\omega^{j}_{k-1,t-1}}(\omega))$, 
where the step-size $\mathtt{lr}$ is chosen by grid search from the set $\{0.1,0.01,0.004,0.008\}$.

\textbf{(ii)} Let $\mathcal{D}$ be the replay memory, in this section, we set the memory size $|\mathcal{D}|=10^{6}$.
For each pair $(s,a)\sim\mathcal{D}$, we conduct the following actor loss
\begin{flalign}
 L_{\theta}(s,a)=-\log\pi_{\theta}(s,a)\underbrace{(\min_{j=1,2}Q_{\omega^{j}_{k-1}}(s,a))}_{\text{Double Q-Learning~\cite{van2016deep}}}
\end{flalign} 
to replace $J(\theta)$.
Then, we calculate the ``noise'' gradient $\dfrac{1}{N_2}\sum_{j=1}^{N_2}(-g(\tau_j|\theta_{k,t})+g(\tau_j|\theta_{k,t-1}))$ as the following $\delta_{k,t}$
\begin{flalign}
\delta_{k,t}=\dfrac{1}{N_2}\sum_{i=1}^{N_2}(\nabla_{\theta}L_{\theta_{k,t}}(s_i,a_i)-\nabla_{\theta}L_{\theta_{k,t-1}}(s_i,a_i)),
\end{flalign}
where $\{(s_i,a_i)\}_{i=1}^{N_2}\sim\mathcal{D}$. For more details, please see $\mathtt{Line~13}$-$\mathtt{16}$ of Algorithm \ref{on-line VRMPO}.

For the implementation of policy $\pi$, we use a Gaussian estimator as follows,
\[\pi_{\theta}(a|s)=\dfrac{1}{\sigma_{\theta}\sqrt{2\pi}}\exp(-\dfrac{a-\mu_{\theta}(s)}{2\sigma^{2}_{\theta}}),\] 
where the logarithmic standard deviation estimator $\log \sigma_{\theta}$ as follows, we use a two layer feedforward neural network of 400 and 300 hidden nodes respectively, with rectified linear units (ReLU) between each layer, and a final tanh unit produces a scalar $\mathtt{net\_output\_value}$;
then we use \[\log \sigma_{\theta}=\mathtt{LOG\_STD\_MIN}+\dfrac{1}{2}(\mathtt{LOG\_STD\_MAX}-\mathtt{LOG\_STD\_MIN})(\mathtt{net\_output\_value}+1),\]
where $\mathtt{LOG\_STD\_MIN}=-20$ and $\mathtt{LOG\_STD\_MAX}=2$.
$\mu_{\theta}(s)=\mathtt{action\_space.high[0]}*\mathtt{net\_output\_value}$, which makes sure actions are in correct range.
In this step,  we use $\alpha_{k}=0.2$ for the iteration (\ref{app-C-theta-iteration}) and the mirror map $\psi$ is $\ell_2$-norm.

\subsection{E.2. Details of Implementation of Baseline Algorithms}
In this section, we provide all the details of implementation of baseline algorithms.
All algorithms, we set $\gamma=0.99$. For VRMPO, the learning rate is chosen by grid search from the set $\{0.1,0.01,0.004,0.008\}$, batch-size $N=100$.  Memory size $|\mathcal{D}|=10^{6}$. We run 5000 iterations for each epoch.

\textbf{DDPG} For the implementation of DDPG, we also use a two-layer feedforward neural network of 400 and 300 hidden nodes, respectively, with rectified linear units (ReLU) between each layer for both the actor architecture and critic architecture, and a final tanh unit following the output of the actor. The step-size of the actor architecture is $10^{-3}$, step-size of the critic architecture is $10^{-1}$, and batch size is $10^{2}$.

In this experiment of DDPG, we set the number of steps of interaction (state-action pairs) for the agent and the environment in each epoch to be 5000.
The replay size is $10^6$. To help exploration, we store enough data to train a model in the replay, and the starting time of training is 10000.
$\gamma=0.99$, the maximum episode is 1000. Both target network are updated with soft update $ \kappa= 0.005$

\textbf{TD3}  
For our implementation of TD3, we refer to the work \cite{fujimoto2018addressing} and \url{https://github.com/sfujim/TD3}.

We excerpt some necessary details about the implementation of TD3~\cite{fujimoto2018addressing}.
TD3 maintains a pair of critics along with a single actor. For each time step, we update the pair of critics towards the minimum target value of actions selected by the target policy:
\begin{flalign}
\nonumber
y &= r + \gamma \min_{i=1,2} Q_{\theta^{'}_i}(s^{'}, \pi_{\phi'}(s') + \epsilon), \\
\nonumber
& \epsilon \sim \text{clip}(\mathcal{N}(0, \sigma), -c, c).
\end{flalign}
Every $d$ iterations, the policy is updated with respect to $ Q_{\theta_1}$ following the deterministic policy gradient algorithm.
The target policy smoothing is implemented by adding $\epsilon \sim \mathcal{N}(0, 0.2)$ to the actions chosen by the target actor network, clipped to $(-0.5, 0.5)$, delayed policy updates consists of only updating the actor and target critic network every $d$ iterations, with $d=2$. While a larger $d$ would result in a larger benefit with respect to accumulating errors, for a fair comparison, the critics are only trained once per time step, and training the actor for too few iterations would cripple learning. Both target networks are updated with $\tau=0.005$.

\textbf{TRPO}  For implementation of TRPO, we refer an open source \url{https://spinningup.openai.com/en/latest/algorithms/trpo.html}.
Recall the basic problem of TRPO as follows,
\begin{flalign}
\label{problem-trpo}
\theta_{k+1}=\arg\max_{\theta}\mathcal{L}(\theta_k,\theta),~~\text{s.t.}~ D_{\text{KL}}(\pi_{\theta}|\pi_{\theta_k})\leq \delta,
\end{flalign}
where where ${\mathcal L}(\theta_k, \theta)$ is the surrogate advantage, a measure of how policy $\pi_{\theta}$ performs relative to the old policy $\pi_{\theta_k}$ using data from the old policy:  ${\mathcal L}(\theta_k, \theta)=\mathbb{E}_{(s,a)\sim\pi_{\theta_k}}\Big[\dfrac{\pi_{\theta}(a|s)}{\pi_{\theta_{k}}(a|s)}A^{\pi_{\theta_k}}\Big]$, $A^{\pi_{\theta_k}}$ is the advantage function estimator, and $D_{\text{KL}}(\pi_{\theta}|\pi_{\theta_k})$ is Kullback-Leibler (KL) divergence. Usually, to get an answer of (\ref{problem-trpo}) quickly, according to \cite{schulman2015trust}, we consider the following problem to approximate the original problem (\ref{problem-trpo}),
\begin{flalign}
\theta_{k+1}=\arg\max_{\theta} g^{\top}(\theta-\theta_k)~~\text{s.t.}~\dfrac{1}{2}(\theta-\theta_k)^{\top}H(\theta-\theta_k)\leq\delta,
\end{flalign}
where $g$ is a policy estimator, $H$ is the Hessian matrix with respect to $\mathcal{L}(\theta_k,\theta)$.
TRPO adds a modification to this update rule: a backtracking line search,
\[\theta_{k+1} = \theta_k + \alpha^j \sqrt{\dfrac{2 \delta}{g^T H^{-1} g}} H^{-1} g,\]
where $\alpha \in (0,1)$ is the backtracking coefficient, and $j$ is the smallest nonnegative integer such that $\pi_{\theta_{k+1}}$ satisfies the KL constraint and produces a positive surrogate advantage.

For the experiments, we run the parameter $\delta$ in the set $\{10^{-2},2\times10^{-2},4\times10^{-2},8\times10^{-2}\}$. We set $\gamma=0.995$ and the maximum episode to be 1000.
For the implementation of critic network, we use a two-layer feedforward neural network of 64 and 64 hidden nodes, respectively, with tanh activation between each layer, and the learning rate of critic network is $10^{-3}$. The actor is also Gaussian policy as same as our VRMPO in Appendix \ref{app-ex:on-line-vrmpo} and the learning rate of critic network is $10^{-2}$.
We use ADAM to learn both actor network and critic network.
For each  each epoch, we let the agent interact with the environment up to be $5\times10^{3}$. We run the backtracking coefficient in the set in the set $\{10^{-1},2\times10^{-1},4\times10^{-1},8\times10^{-1}\}$.

\textbf{PPO}. For the implementation of PPO, we refer to an open source \url{https://github.com/openai/baselines/tree/master/baselines}.

PPO develops TRPO, and its objective reduces to
\[
L(s,a,\theta_k,\theta) = \min\Big\{
\dfrac{\pi_{\theta}(a|s)}{\pi_{\theta_k}(a|s)}A^{\pi_{\theta_k}}(s,a), g( \epsilon,A^{\pi_{\theta_k}}(s,a))
\Big\},
\] where
$
g(\epsilon, A) =(1+\epsilon)A
$, if $A\ge0$; else, $g(\epsilon, A) =(1-\epsilon)A$.
In this experiment, we use the same actor-critic network as TRPO, and clip parameter $\epsilon=0.2$.

\end{document}